\documentclass{article} % For LaTeX2e
\usepackage[final]{colm/colm2026_conference}

\usepackage{microtype}
\usepackage{hyperref}
\usepackage{url}
\usepackage{booktabs}

\usepackage{lineno}

\definecolor{darkblue}{rgb}{0, 0, 0.5}
\hypersetup{colorlinks=true, citecolor=darkblue, linkcolor=darkblue, urlcolor=darkblue}

\usepackage{graphicx}
\usepackage{subcaption}
\usepackage[dvipsnames]{xcolor}

\usepackage{amsmath}
\usepackage{physics}
\usepackage{wrapfig}
\usepackage{amssymb}
\usepackage{mathtools}
\usepackage{amsthm}
\usepackage{tikz}
\usetikzlibrary{positioning,arrows.meta,calc}

\definecolor{bauhausblue}{HTML}{0061AC}
\definecolor{bauhausred}{HTML}{DF232C}
\definecolor{bauhausyellow}{HTML}{F1B50E}

\usepackage{listings}
\usepackage[textsize=tiny,textwidth=3cm]{todonotes}

\usepackage[capitalize,noabbrev]{cleveref}

\theoremstyle{plain}
\newtheorem{theorem}{Theorem}[section]

\theoremstyle{definition}

\newtheorem{principle}[theorem]{Principle}
\theoremstyle{remark}

\title{Superposition Without Interference? \\ Towards Isolated Interventions via \\ Almost Orthogonal Features in Language Models}

\author{Moritz Miller$^{124}$\thanks{ The authors contributed equally to this work. Author order was determined by a 60–40 coin flip.} , Florent Draye$^{13\star}$ \& Bernhard Schölkopf$^{1234}$ \\
$^1$ Max Planck Institute for Intelligent Systems, Tübingen, Germany \\
$^2$ ETH Zurich, Switzerland \\
$^3$ ELLIS Institute Tübingen, Germany \\
$^4$ Max Planck ETH Center for Learning Systems \\
\texttt{moritz.miller@tuebingen.mpg.de} % \\
}

\newcommand{\TopK}{\mathrm{TopK}}

\newcommand{\K}{\mathrm{K}}

\newcommand{\tril}{\mathrm{tril}}

\newcommand{\Atoms}{\{\vf_j\}_{j \in [d]}}
\newcommand{\AtomsBracket}{\left[ \vf_1,...,\vf_d \right]}
\newcommand{\Frob}{\text{F}}
\newcommand{\GemmaTwoB}{Gemma~2~2B }
\newcommand{\LlamaThree}{Llama~3.1~8B-Instruct }
\newcommand{\Llama}{Llama~3.2~1B }
\newcommand{\normvec}[1]{#1}
\newcommand{\gsmk}{\texttt{GSM8K}}
\newcommand{\mmQA}{\texttt{MetaMathQA}}

\newcommand{\Rmodel}{\mathbb{R}^m}
\newcommand{\Rdict}{\mathbb{R}^d}
\newcommand{\Rdm}{\mathbb{R}^{d \times m}}
\newcommand{\Rmd}{\mathbb{R}^{m \times d}}

\newcommand{\Rmm}{\mathbb{R}^{m \times m}}

\newcommand{\vxhat}{\hat{\vx}}

\newcommand{\vxtilde}{\Tilde{\vx}}

\newcommand{\biasD}{\vb_\mD}
\newcommand{\biasE}{\vb_\mE}
\newcommand{\transpose}{\top}
\newcommand{\zeros}{\mathbf{0}}

\newcommand{\doop}{\mathrm{do}}

\newcommand{\NoiseDist}{\mathbb{P}_\mN}
\newcommand{\IntervMech}{\widetilde{f_l}(\widetilde{\Pa}_l, \widetilde{N_l})}
\newcommand{\SCMDist}{\mathbb{P}_\mX}
\newcommand{\SCMMarginals}[1]{P({X_{#1}}|\Pa_{#1})}
\newcommand{\IntervMarginals}[1]{\widetilde{P}({X_{#1}}|\Pa_{#1})}

\newcommand{\PZ}{\mathbb{P}_Z}

\newcommand{\aqua}{\textbf{\textcolor{bauhausblue}{aqua}}}
\newcommand{\Aquaman}{\textbf{\textcolor{bauhausblue}{Aquaman}}}

\newcommand{\Mike}{\textbf{\textcolor{bauhausyellow}{Mike}}}

\newcommand{\high}{\text{\textcolor{bauhausred}{high}}}
\newcommand{\medium}{\text{\textcolor{bauhausyellow}{medium}}}
\newcommand{\low}{\text{\textcolor{bauhausblue}{low}}}

\newcommand{\animal}{\texttt{animal}}

\newcommand{\location}{\texttt{location}}

\newcommand{\ROUGE}{\textsc{ROUGE-L recall}}
\newcommand{\LCS}{\textsc{LCS}}
\newcommand{\JSD}{D_{\mathrm{JSD}}}
\newcommand{\mymacro}[1]{{#1}}

\newcommand{\veta}{{\mymacro{\boldsymbol{\eta}}}}
\newcommand{\vgamma}{{\mymacro{\boldsymbol{\gamma}}}}

\newcommand{\Pa}{\mathrm{Pa}}

\newcommand\smalldots{\hbox to 1em{.\hss.\hss.}}

\newcommand{\loss}{{\mymacro{ \ell}}}

\newcommand{\negterm}[1]{{\mymacro{ {\raise.17ex\hbox{$\scriptstyle\sim$}} #1}}}

\newcommand{\ignore}[1]{}
\newcommand{\expandLater}[1]{}

\def\1{\mathbf{1}}

\def\vb{{{\mymacro{ \mathbf{b}}}}}

\def\ve{{{\mymacro{ \mathbf{e}}}}}
\def\vf{{{\mymacro{ \mathbf{f}}}}}

\def\vs{{{\mymacro{ \mathbf{s}}}}}

\def\vu{{{\mymacro{ \mathbf{u}}}}}
\def\vv{{{\mymacro{ \mathbf{v}}}}}

\def\vx{{{\mymacro{ \mathbf{x}}}}}
\def\vy{{{\mymacro{ \mathbf{y}}}}}
\def\vz{{{\mymacro{ \mathbf{z}}}}}

\def\mA{{{\mymacro{ \mathbf{A}}}}}

\def\mD{{{\mymacro{ \mathbf{D}}}}}
\def\mE{{{\mymacro{ \mathbf{E}}}}}

\def\mG{{{\mymacro{ \mathbf{G}}}}}

\def\mI{{{\mymacro{ \mathbf{I}}}}}

\def\mN{{{\mymacro{ \mathbf{N}}}}}

\def\mX{{{\mymacro{ \mathbf{X}}}}}

\newcommand{\N}{{\mymacro{ \mathbb{N}}}}

\newcommand{\KL}{{\mymacro{ D_{\mathrm{KL}}}}}

\DeclareMathSymbol{\mlq}{\mathord}{operators}{``} %math left quote
\DeclareMathSymbol{\mrq}{\mathord}{operators}{`'} %math right quote

\begin{document}

\ifcolmsubmission
\linenumbers
\fi

\maketitle

\begin{abstract}
A central premise in mechanistic interpretability is that meaningful concepts in language models are represented by linear features in activation space. For such features to support reliable interventions, manipulating one feature should not substantially alter the effects of others. In practice, however, feature entanglement leads to interference such that localized interventions can have unintended downstream effects. Motivated by the \textit{Independent Causal Mechanisms} principle, we propose to constrain internal features to be almost orthogonal. We argue that this promotes modular representations amenable to causal intervention. We formalize this problem by characterizing the gap between an idealized isolated intervention and its realized effect on model outputs in terms of feature interference. We upper-bound the propagation of feature interference in terms of the self-coherence of the feature dictionary, and relate this discrepancy to an explicit orthogonality regularization on the dictionary itself. Empirically, we show that this regularization enables more isolated interventions on mathematical reasoning concepts while preserving model performance. Our code is available under \texttt{https://github.com/mrtzmllr/sae-icm}.
\end{abstract}
% In addition, we upper-bound the feature interference in terms of the dictionary’s self-coherence. We then show that restricting the Jacobian is equivalent to a direct orthogonality regularization on the dictionary. 
\begin{figure*}[h]
     \centering
     \begin{subfigure}[b]{\linewidth}
         \centering
         \begin{lstlisting}[frame=none]
|b|Question:|
|y|Mike| writes a 3-page letter to 2 different friends twice a week. How
many pages does he write a year?
\end{lstlisting}
     \end{subfigure}
    % \hfill
    \begin{subfigure}[b]{0.48\linewidth}
        \centering
         \begin{lstlisting}[frame=none]
|b|Standard Generation:|
|y|Mike| writes a 3-page letter
to 2 different friends, so 
he writes 3 x 2 = |r|6| 
pages per week.

He writes this letter twice
a week, so he writes 6 x 2 = |r|12| 
pages per week.

There are 52 weeks in a year,
so |y|Mike| writes 12 x 52 = 624
pages in a year.

#### 624
The answer is: |r|624|
\end{lstlisting}
    \end{subfigure}
    \hfill
    \begin{subfigure}[b]{0.48\linewidth}
        \centering
         \begin{lstlisting}[frame=none]
|b|Post-Intervention Generation:|
If |u|Aquaman| writes a 3-page letter
to 2 different friends, he writes 
a total of 3 x 2 = |r|6| 
pages per week.

Since he writes this letter twice 
a week, he writes a total
of 6 x 2 = |r|12| pages per week.

If there are 52 weeks in a year, 
then he writes a total of 
12 x 52 = 624 pages in a year.

#### 624
The answer is: |r|624|
\end{lstlisting}
    \end{subfigure}
\caption{\textbf{Intervention on SAE feature } After regularizing the SAE decoder under an orthogonality penalty, we perform an isolated intervention on the feature associated with $\Mike$ by exchanging it with the feature corresponding to the prefix $\aqua$. The model accordingly substitutes $\Mike$ for $\Aquaman$ at inference while maintaining its reasoning capabilities. \looseness=-1}
\label{fig:main}
\end{figure*}
\section{Introduction}

The \textit{Independent Causal Mechanisms} (ICM) principle posits that complex systems can be decomposed into autonomous modules that do not inform or influence each other~\citep{Janzing2010CausalCondition}. In the context of learned representations, this implies that intervening on one latent mechanism should not systematically alter the functional form of the remaining mechanisms. This perspective is especially appealing for mechanistic interpretability. If a feature corresponds to a meaningful concept, an intervention on that feature should induce a targeted change in model behavior without substantially altering the effects of other features. More broadly, this would move us closer to representations in which a conceptual intervention in the output corresponds to a localized intervention inside the model.

In practice, however, learned features in language models are often not isolated. Features that are strongly aligned in representation space can interfere with one another, so that modifying one feature in the residual stream also changes the influence of others~\citep{elhage2022superposition}. This makes the effect of steering ambiguous, blurs the causal contribution of the targeted feature, and limits intervenability within the model. Prior work on superposition further suggests that such feature alignment can induce interference patterns that are difficult to control and may even support adversarial perturbations~\citep{gorton2025adversarial}. In this paper, we formalize feature interference and study how reducing it improves the isolation of feature-level interventions.

Sparse autoencoders (SAEs) provide a natural setting in which to address this problem. They have become a central tool in mechanistic interpretability because they uncover large numbers of human-interpretable features in language model representations~\citep{cunningham2023SparseAF,paulo2024automatically}. In our setting, SAEs serve not only to discover such features, but also to shape the geometry of the feature space itself. Concretely, we fine-tune an SAE with an orthogonality penalty, freeze the resulting weights, and then adapt the language model around it using low-rank adaptation, while preserving overall performance~\citep{hu2021lora,wang2025resatransparentreasoningmodels,checn2025lowrankadapting}. This yields features that remain interpretable while exhibiting substantially less interference, allowing more localized interventions with limited spill-over onto other features. \autoref{fig:main} illustrates this setup: after replacing the feature corresponding to the concept $\Mike$ with one associated with usages of the prefix $\aqua$, the model still reasons correctly toward $\mathbf{\textcolor{bauhausred}{624}}$ while adapting the intervention coherently to context, replacing the male first name $\Mike$ with $\Aquaman$ during generation. Despite this local concept substitution, the downstream effect on other concepts remains limited.

Given this methodology, our contributions are:

\begin{itemize}
    \item We theoretically relate intervenability in the residual stream to the interference between features (Section~\ref{sec:intervenabilitytheory}). By a causal representation argument under a linear hypothesis, we upper-bound the expected discrepancy between true and model output in terms of the representation-space interference. Explicitly controlling the latter, we obtain bounds for the former.
    \item We fine-tune a LM around a fine-tuned SAE whose decoder matrix is regularized to be almost orthogonal. We demonstrate that this fine-tuning pipeline maintains similar performance and interpretability scores on mathematical reasoning as a non-penalized SAE (Section~\ref{sec:exportho}).
    \item To the best of our knowledge, we are the first to perform \textit{isolated interventions} on specific concepts inside the SAE by explicitly limiting spill-over effects on other features (Section~\ref{sec:expinterventions}). We demonstrate that the increased orthogonality penalty improves intervenability in the model's representation space.
    % \item We show that this setup obtains similar interpretability scores as non-penalized SAEs (Section~\ref{sec:expinterp}). More importantly, Section~\ref{sec:expembeddings} reflects that the embedded feature explanations are significantly less similar to each other under our orthogonality penalty.
    % \item We theoretically relate intervenability in the residual stream to the interference between features. Leveraging finite frame theory, we argue that interference between features increases similarity between features (Section~\ref{sec:intervenabilitytheory}). \looseness=-1
\end{itemize}
\section{Related work}
\paragraph{Intervenability}
In applied causality, a set of interventions is commonly described as datasets from different environments~\citep{icp}. Under this notion, one main objective lies in identifying the causal mechanisms prevalent across environments~\citep{dofinetti}. Additionally, the unique disentanglement of hidden concepts appears in causal representation learning~\citep{bernhard2021towardcrl}. Here, identifiability concerns finding a unique model which can explain the observed data. While non-identifiability in causal representation learning represents a well-studied problem~\citep{HYVARINEN1999ica,locatello2019challenging}, additional information such as access to interventional data can yield identifiability results~\citep{buchholz2023learninglinear}. Thus, traditional causality studies interventions as distributional properties of the underlying data-generating process~\citep{Spirtes2000causation}. Recent work utilizes the SAE to infer information about the data~\citep{mencattini2026exploratorycausalinferencesaence}. In contrast, we study intervenability on the modeling level. This is similar to uncovering the causal mechanisms inside the model's representation space~\citep{iccr}. A general causal abstraction perspective on mechanistic interpretability is provided by~\citet{Geiger2023CausalAA}. \looseness=-1 % Combining those two facets of intervenability, recent work utilizes the SAE inside the model to infer information about the data~\citep{mencattini2026exploratorycausalinferencesaence}.

%\paragraph{Identifiability}\moritz{drop this section, move CRL to interventions}
%Identifiability in sparse dictionary learning is centered around the uniqueness bound in~\autoref{thm:self-coherence} and low self-coherence in general. Beyond conventional dictionary learning, self-coherence~\citep{nejati2016coherence} represents a relevant guiding principle in signal processing~\citep{sigg2012learning} and computer vision~\citep{hawe2013separable}. 
% In addition to the aforementioned literature, the unique disentanglement of hidden concepts reappears in causal representation learning~\citep{bernhard2021towardcrl}. Here, identifiability concerns finding a unique model which can explain the observed data. While non-identifiability in causal representation learning represents a well-studied problem~\citep{HYVARINEN1999ica,locatello2019challenging}, additional information such as access to interventional data can yield identifiability results~\citep{buchholz2023learninglinear}.

\paragraph{Interpretability}
The linear superposition hypothesis~\citep{elhage2022superposition} posits that LMs represent features as directions in their activation space. In this way, the model can represent more features than it has dimensions. To disentangle the polysemantic residual stream into human-interpretable features, SAEs reconstruct the activations as sparse linear combinations of features~\citep{cunningham2023SparseAF}. Further, the linear representation hypothesis~\citep{park2024lrh} hinges on the observation that semantically unrelated features are represented as almost orthogonal vectors in some representation space~\citep{jiang2024ontheorigins}. Moreover, recent empirical evidence supports the emergence and test-time use of such linear structure in LMs~\citep{chen2026emergence}. In SAEs specifically, we empirically observe high self-coherence between features such that finding a canonical dictionary is impossible~\citep{marks2024enhancing,leask2025sparse}. % Beyond language modeling, SAEs are employed in genomics~\citep{Pedrocchi2025sparse,Maiwald2025decodeglm} and computer vision~\citep{pach2025sparseautoencoderslearnmonosemantic}. 
In the context of sparsity, weight-sparse models~\citep{gao2025weight} and sparse attention fine-tuning~\citep{draye2025sparse} share our objective of constructing inherently interpretable architectures. \looseness=-1

\section{Background}

\paragraph{Notation}
We employ an abbreviated set notation and write $[n] := \{1,...,n\}$ for integer $n$.

\subsection{Causality}
This section follows~\citet{elements}. For additional details, we refer the reader there. The structural causal model (SCM) consists of a collection of $m$ structural assignments
\begin{align*}
    X_i := f_i(\Pa_i, N_i), i \in [m]
\end{align*}
for $\Pa_i \subseteq \{X_1, ..., X_m\} \backslash \{X_i\}$ the parents of $X_i$, and the distribution $\NoiseDist = \mathbb{P}_{N_1,...,N_m}$ over the jointly independent noise variables~\citep{pearl2009}. It entails a joint distribution $\SCMDist$ over $\{X_1, ...,X_m\}$ as well as the marginals $\SCMMarginals{i}$ for $i \in [m]$. An intervention in the SCM is the replacement of at least one structural assignment, $l \in [m]$, % to obtain the new SCM $\SCMInterv$, 
\begin{align*}
    X_l := \IntervMech
\end{align*}

inducing $\IntervMarginals{l}$. We denote the intervention by $\doop(X_l := \IntervMech)$.
% \moritz{maybe add additivity} Graphically, we represent an SCM as a directed acyclic graph with $\{X_1, ..., X_m\}$ as nodes and directed edges from each causal parent of $i$ to $i$. 
For any $i \in [m]$, $\SCMMarginals{i}$ describes the causal mechanism generating $X_i$. With this, we can state the ICM principle~\citep{Janzing2010CausalCondition,causaldefinetti}. \looseness=-1
\begin{principle}[Independent Causal Mechanisms]
    Causal mechanisms are independent of each other in the sense that a change in one mechanism $\SCMMarginals{l}$ does not inform or influence any other mechanism $\SCMMarginals{i}$, for $i \neq l$. \looseness=-1
\label{def:icm}
\end{principle}

\subsection{Sparse autoencoders}
\label{sec:sae}
SAEs are trained on reconstructing the activations in the residual stream under a sparsity penalty. The output is then a sparse linear combination of vectors. Let $\vx \in \Rmodel$ denote the Transformer state prior to the SAE. A $\TopK$ SAE~\citep{makhzani2013ksae,gao2025topk} then outputs the reconstruction $\vxhat \in \Rmodel$,
\begin{align*}
    \vz &= \TopK \left( \mE \vx + \biasE \right) \\
    \vxhat &= \mD\vz + \biasD
\end{align*}
for $\vz \in \Rdict$ the coefficients for the features and $\mE \in \Rdm, \mD \in \Rmd$ encoder and decoder matrices with biases $\biasE \in \Rdict, \biasD \in \Rmodel$, respectively. We call $\mD$ our feature dictionary with $d \gg m$ features.

We train our autoencoder on the normalized reconstruction loss with penalty on $\mD^\transpose \mD$,
\begin{align}
    \loss ( \vx ) = \frac{\|\vxhat - \vx\|_2}{\|\vx\|_2} + \lambda \|\tril (\mD^\transpose \mD)\|_\Frob^2
\label{eqn:reconstruction}
\end{align}
where $\|\tril(\mD^\transpose \mD)\|_\Frob^2$ enforces almost orthogonality between features with $\lambda > 0$. By $\tril(\mD^\transpose \mD)$ we denote the off-diagonal lower-triangular elements of $\mD^\transpose \mD$. During the forward pass, $\TopK$ only activates the features associated with the $\K$ highest coefficients with all other activations zeroed out. \looseness=-1

\subsection{Sparse dictionary learning}
Similar to the features in the SAE decoder, the atoms compose the dictionary in sparse dictionary learning. If the number of atoms in the dictionary exceeds the number of dimensions per atom, we call the dictionary \textit{overcomplete}. Thus, for $\mD \in \Rmd$ with $d > m$, the atoms form an overcomplete spanning set. In the overcomplete setting, self-coherence determines how sparse the representations must be for identifiability to hold~\citep{donoho2005stable}. Relevant to our later work, we define the mean self-coherence here as

\begin{align}
    % \mu(\mD) = \max_{i \neq j} |\langle \vf_i, \vf_j \rangle|.
    \mu(\mD):=\frac{1}{d(d-1)}\sum_{j = 1}^n\sum_{k\neq j}\langle \vf_j,\vf_k\rangle^2.
\label{eqn:self-coherence}
\end{align}

Welch bounds provide lower bounds on the similarity between any two atoms~\citep{welch1974lowerbounds}. % With respect to~\eqref{eqn:self-coherence}, we can derive the lower bound for unit-norm atoms. 
Let $\Atoms$ unit-norm for $\vf_j \in \Rmodel$ denote the $d$ atoms composing the dictionary. For $\N_{>0}$ non-zero natural numbers and $s \in \N_{>0}$, we have
\begin{align}
    \frac{1}{d^2} \sum_{j = 1}^d \sum_{k = 1}^d |\langle \normvec{\vf_j}, \normvec{\vf_k} \rangle|^{2s} \geq \frac{1}{\binom{m+s-1}{s}}.
        % \frac{\sum_{i = 1}^d \sum_{j = 1}^d |\langle \vf_i, \vf_j \rangle|^{2s}}{\left(\sum_{i = 1}^d \|\vf_i\|\right)^{2s}} \geq \frac{1}{\binom{m+s-1}{s}}.
\label{eqn:generalizedwelch}
\end{align}
In Appendix~\ref{app:welch}, we discuss self-coherence for identifiability and add further Welch bounds. \looseness=-1

% ICML
% \subsection{Finite Frame Theory}
% For $\Hilbert$ Hilbert space and $\Frame$ a set of vectors with $\vf_j \in \Rmodel, j \in [d]$, we call the finite-dimensional set $\Frame$ a frame for $\Hilbert$ iff $\Span(\Frame) = \Hilbert$~\citep{Casazza2013finiteframe}. The analysis operator $\mT: \Hilbert \longrightarrow \ell_2^d$ is then defined as
% \begin{align}
%     \mT \vx := \left\{ \langle \vx, \vf_j \rangle \right\}_{j \in [d]}
% \label{eqn:analysisoperator}
% \end{align}
% for $\vx \in \Hilbert$. It follows that the adjoint operator $\mT^*: \ell_2^d \longrightarrow \Hilbert$ of $\mT$ is given for $\vz \in \Rdict$ by
% \begin{align}
%     \mT^* (\vz) = \sum_{j = 1}^d z_j \vf_j.
% \label{eqn:adjointoperator}
% \end{align}

% A frame is overcomplete if $d > m$.
\section{Aligned features reduce intervenability}
\label{sec:intervenabilitytheory}

We consider the causal graph
\[
\begin{array}{ccc}
Z & \longrightarrow & Y \\
\downarrow &  & \\
X & \longrightarrow & \widehat Y
\end{array}
\]
for \(Z\in\mathbb{R}^d\) the latent state, \(X\in\mathbb{R}^m\) the model representation, \(Y\) the true output, and \(\widehat Y\) the model output. We assume that the model representation is a sparse linear combination of features, i.e., \looseness=-1
\[
X=\mD Z.
\]

In our theoretical setup, we assume that \(Z\) is exogenous with distribution $\PZ$, and, following~\citet{park2024lrh}, we model the true and conditional output distributions as exponential-family distributions with natural parameters
\[
\veta := \mA Z,
\qquad
\widehat{\veta} := \mA \widehat Z,
\]
where \(\mD=\AtomsBracket\in\mathbb{R}^{m\times d}\), \(\mA\in\mathbb{R}^{q\times d}\), and \(\vgamma(\vy)\in\mathbb{R}^q\) is a fixed embedding of the output \(\vy\). The decoded latent representation is
\[
\widehat Z = \mD^\top X = \mD^\top \mD Z = \mG Z.
\]
Accordingly, the true and model output distributions are given by
\[
P\bigl(Y=\vy\mid Z\bigr)\propto \exp\!\bigl(\veta^\top \vgamma(\vy)\bigr),
\qquad
P\bigl(\widehat Y=\vy\mid Z\bigr)\propto \exp\!\bigl(\widehat{\veta}^\top \vgamma(\vy)\bigr).
\]
Thus, the true output depends on the latent state through \(\veta=\mA Z\), whereas the model output depends on the decoded representation through \(\widehat{\veta}=\mA\mG Z\).

Now fix \(j\in[d]\) and \(\delta\in\mathbb{R}\), and consider the intervention
\[
\doop(Z_j:=Z_j+\delta),
\]
which yields the intervened latent state
\[
Z^{(j,\delta)} := Z+\delta \ve_j.
\]
Let \(Y^{(j,\delta)}\) and \(\widehat Y^{(j,\delta)}\) denote the corresponding post-intervention output probability vectors.
We measure the discrepancy between the true and model post-intervention outputs by
\begin{align}
\mathbb{E}\bigl[\|Y^{(j,\delta)}-\widehat Y^{(j,\delta)}\|_2^2\bigr].
\label{eqn:discrepancy}
\end{align}
and upper-bound the expected discrepancy at the output level. We defer all proofs and details to Appendix~\ref{app:interference}.
\begin{theorem}[Upper bound on interference]
Let $\mD=\AtomsBracket \in \Rmd$ and assume that
\[
\|\vf_j\|_2=1,\qquad j \in [d].
\]
% Fix $j \in [d]$ and \(\delta\in\mathbb{R}\). Then
% \[
% \mathbb{E}\Bigl[\|Y^{(j,\delta)}-\widehat Y^{(j,\delta)}\|_2^2\Bigr]
% \le
% L^2\|\mA\|_{\mathrm{op}}^2\,
% \mathbb{E}\Bigl[\|(\mG-\mI)(Z+\delta \ve_j)\|_2^2\Bigr].
% \]
% Consequently,
% \[
% \mathbb{E}\Bigl[\|Y^{(j,\delta)}-\widehat Y^{(j,\delta)}\|_2^2\Bigr]
% \le
% 2L^2\|\mA\|_{\mathrm{op}}^2
% \Bigl(
% \mathbb{E}[\|(\mG-\mI)Z\|_2^2]
% +
% \delta^2\sum_{k\neq j}\langle \vf_k,\vf_j\rangle^2
% \Bigr).
% \]
% Averaging over \(j\) yields
Then, for $\mu(\mD)$ as defined in~\eqref{eqn:self-coherence},
\[
\frac1d\sum_{j=1}^d
\mathbb{E}\Bigl[\|Y^{(j,\delta)}-\widehat Y^{(j,\delta)}\|_2^2\Bigr]
\le
\frac{1}{2}\|\mA\|_{\mathrm{op}}^2\,
\mathbb{E}[\|(\mG-\mI)Z\|_2^2]
+
\frac{1}{2}\|\mA\|_{\mathrm{op}}^2\,\delta^2(d-1)\mu(\mD).
\]
% \[
% \mu(\mD):=\frac{1}{d(d-1)}\sum_{j\neq k}\langle \vf_j,\vf_k\rangle^2.
% \]
\label{thm:upper-bound}
\end{theorem}

The theorem shows that the post-intervention mismatch is controlled by a baseline term,
\[
\mathbb E[\|(\mG-\mI)Z\|_2^2],
\]
and an intervention-specific term,
\[
\delta^2\sum_{k\neq j}\langle \vf_k,\vf_j\rangle^2.
\]
Thus, when the columns of \(\mD\) are almost orthogonal, interventions on a latent feature induce only limited distortion in the model output. This becomes explicit under additional assumptions on the latent distribution. Suppose that \(Z\) has uniformly random support of size \(\K\), and that conditioned on its support the nonzero coordinates are independent, mean zero, and have variance \(\sigma^2\). Then, the averaged bound of the theorem yields
% \[
% \frac1d\sum_{j=1}^d
% \mathbb E\Bigl[\|Y^{(j,\delta)}-\widehat Y^{(j,\delta)}\|_2^2\Bigr]
% \le
% 2L^2\|\mA\|_{\mathrm{op}}^2(d-1)\mu(\mD)\bigl(\K\sigma^2+\delta^2\bigr).
% \]
% In particular, 

\[
\frac1d\sum_{j=1}^d
\mathbb E\Bigl[\|Y^{(j,\delta)}-\widehat Y^{(j,\delta)}\|_2^2\Bigr]
\le
\frac12\|\mA\|_{\mathrm{op}}^2(d-1)\mu(\mD)\bigl(\K\sigma^2+\delta^2\bigr).
\]

\paragraph{Empirical investigation}
\begin{wrapfigure}{r}{0.5\linewidth}
    \centering
    \includegraphics[width=\linewidth]{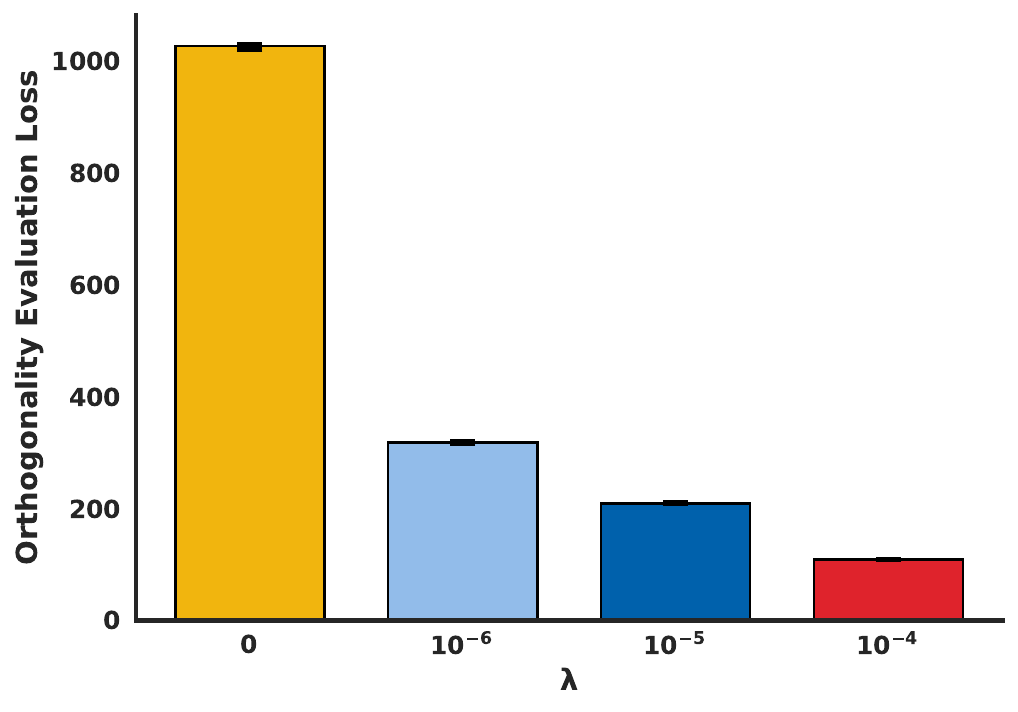}
    \caption{\textbf{Orthogonality evaluation loss } We plot the orthogonality loss $\|\tril(\mD^\transpose \mD)\|_\Frob^2$ for all values of $\lambda$. For each $\lambda$, we evaluate on a subset of $1'024$ active features in the decoder. Error bars represent confidence intervals obtained from running $100$ evaluations. \looseness=-1}
    \label{fig:orthogonality}
\end{wrapfigure}
In this regime, both the baseline term and the intervention-specific term are governed by the same coherence quantity \(\mu(\mD)\), making the role of decoder geometry fully explicit.
We design a short numerical example to demonstrate the effect of interference between features on the expected post-intervention mismatch. For $\mA$ random matrix with standard Gaussian entries, we consider the case where $q > d > m$. If we understand $q$ as the output dimension similar to the vocabulary size in language modeling, this follows the same hierarchy as in our SAE experiments. We compare an implementation with L2-normalized random Gaussian columns to settings with higher and lower self-coherence. Confirming our intuition, we observe that low self-coherence decreases the intervention mismatch. Details and results we defer to Appendix~\ref{app:interference}.
\section{Experiments}
\label{sec:experiments}
Recent work demonstrates that fixing an SAE inside the LM during fine-tuning achieves low cross-entropy~\citep{checn2025lowrankadapting}. We extend this two-step procedure~\citep{wang2025resatransparentreasoningmodels} by restricting the SAE decoder to be almost orthogonal and fine-tuning the LM around the SAE. We insert the SAE after Transformer layer $13$ of $26$. In doing so, we add to the previous section by providing evidence that orthogonality regularization at an intermediate layer also improves intervenability. \looseness=-1

\begin{enumerate}
    \item We optimize a pre-trained $\TopK$ SAE~\citep{gao2025topk} for low self-reconstruction under an orthogonality penalty~\eqref{eqn:reconstruction} on the decoder matrix.
    \item Fixing the SAE decoder, we then insert the SAE into a \GemmaTwoB Transformer and low-rank adapt~\citep{hu2021lora} the LM on cross-entropy.
\end{enumerate}

\label{sec:exportho}
In total, we fine-tune four SAEs with orthogonality penalty $\lambda \in \{0, 10^{-6}, 10^{-5}, 10^{-4}\}$. In particular, we demonstrate all our claims in comparison to the non-penalized case of $\lambda = 0$. Throughout, we choose $\K = 20$. Both steps in our pipeline involve fine-tuning the model for one epoch on $\mmQA$~\citep{yu2023metamath}. First, we fully fine-tune the SAE. Next, we low-rank adapt the attention matrices including the non-decoder SAE weights. Therefore, we restrict the decoder to contain almost orthogonal features, while allowing the rest of the architecture to adapt around this penalty. Adopting the SAEBench~\citep{karvonen2025saebenchcomprehensivebenchmarksparse} module of size $2^{16} = 65'536$, we input the SAE into the $26$-layer residual stream after layer $13$. Thus, all information flows through the SAE. To limit the computational overhead induced by the penalty in~\eqref{eqn:reconstruction}, we evaluate the decoder regularization at every step on a subset of $2^{10} = 1'024$ randomly drawn features. This makes fine-tuning feasible while updating decoder weights progressively. All details including hyperparameter choice and training duration are deferred to Appendix~\ref{app:experiments}.

\subsection{Almost orthogonality while keeping performance}

% The generalized Welch bound in~\eqref{eqn:generalizedwelch} provides a measure of comparing the average similarity between dictionary features across SAEs. Evaluating 
% \begin{align}
%     \frac{\sum_{i = 1}^d \sum_{j = 1}^d |\langle \vf_i, \vf_j \rangle|^{2s}}{\left(\sum_{i = 1}^d \|\vf_i\|\right)^{2s}}
% \label{eqn:generalizedwelchmetric}
% \end{align}
% for $s = 1$, we obtain a standardized comparison metric to study the effect of penalizing the SAE. 
\autoref{fig:orthogonality} displays the orthogonality loss $\|\tril(\mD^\transpose \mD)\|_\Frob^2$ for increasing orthogonality penalty. Error bars represent the basic bootstrap~\citep{efron1979basicbootstrap} confidence intervals at the $95\%$-level for $100$ sampled datasets. For computational efficiency, we compute the evaluation loss on $100$ draws of $2^{10} = 1'024$ randomly selected active features. That is, we only evaluate on features that activate at least once in the test set. First, we observe that no penalty yields the highest orthogonality loss. This confirms the intuition that features do not empirically tend toward orthogonality if not explicitly regularized to do so. Further, we observe monotonically decreasing loss as $\lambda$ increases. Inserting the orthogonality penalty, therefore, has the desired effect. \looseness=-1

\begin{wrapfigure}{r}{0.5\linewidth}
    \centering
    \includegraphics[width=\linewidth]{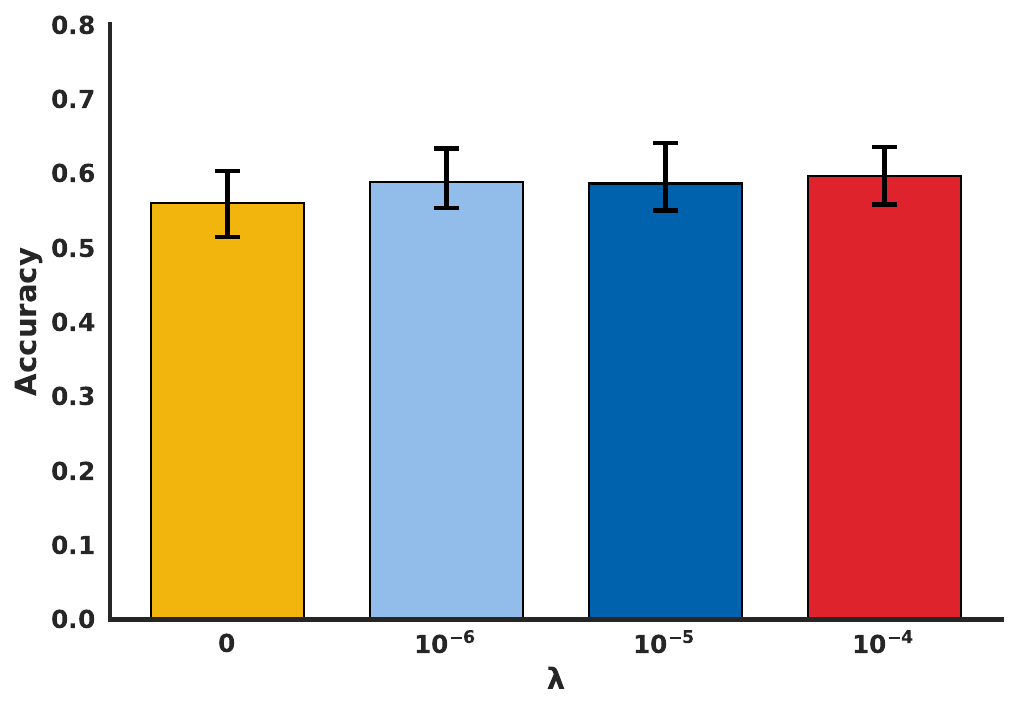}
    \caption{\textbf{Evaluation on \gsmk } We evaluate on the $\gsmk$ test set. Error bars represent the basic bootstrap confidence intervals~\citep{efron1979basicbootstrap} on $100$ randomly drawn datasets.}
    \label{fig:mathevaluation}
\end{wrapfigure}
Next, we evaluate the fine-tuned LM on the $\gsmk$~\citep{cobbe2021gsm8k} test set. \textit{Inter alia}, $\gsmk$'s training set is augmented in $\mmQA$. \citet{yu2023metamath} observe performance in the range of $[0.665, 0.777]$ for $7$B models fine-tuned on $\mmQA$. In our two-step fine-tuned $2$B LM, we obtain competitive performance for increasing orthogonality penalty $\lambda$. \autoref{fig:mathevaluation} represents the accuracy on the $\gsmk$ test set. While just below the performance of the $7$B models, our $2$B models do not significantly differ in performance for increasingly strict penalty. Again, the evaluation accuracy on the model fine-tuned for all non-zero values of $\lambda$ is comparable to the no-penalty LM. If anything, we observe a slight performance increase in the regularization parameter $\lambda$. Together with our findings from~\autoref{fig:orthogonality}, we show that stricter orthogonality does not come at the cost of performance.

\subsection{Localized interventions work}
\label{sec:expinterventions}

In Section~\ref{sec:intervenabilitytheory}, we show that high self-coherence in the dictionary inhibits isolated intervenability on individual features. To confirm this idea, we intervene locally on known features in the SAE and record the downstream effect during generation. In particular, every time a certain feature is activated in the residual stream, we swap it for a different one. We then check if the model remains capable of reasoning about the overarching problem while adapting its generation output to the exchanged features. 
\begin{figure*}[h]
     \centering
     \begin{subfigure}[b]{0.49\textwidth}
         \centering
         \includegraphics[width=\textwidth]{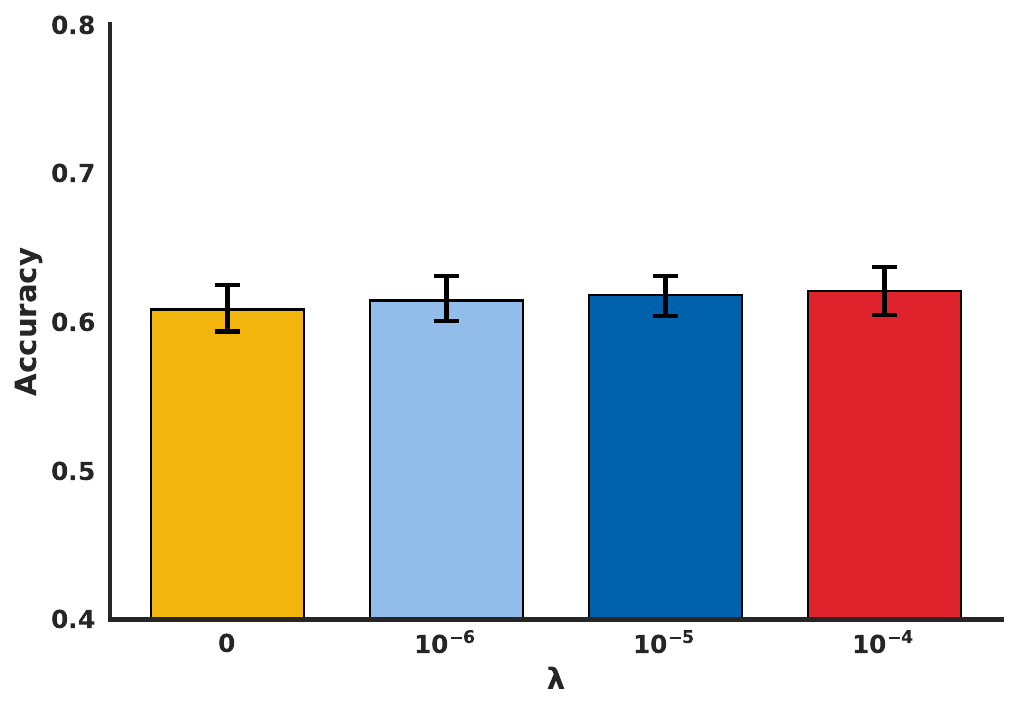}
         \caption{\textbf{Evaluation on mathematical reasoning } We plot mathematical reasoning performance on $3'960$ hand-designed examples after intervening on the SAE. Error bars are basic bootstrap confidence intervals~\citep{efron1979basicbootstrap} with $100$ samples.}
         \label{fig:intervenability_eval}
     \end{subfigure}
     \hfill
     \begin{subfigure}[b]{0.49\textwidth}
         \centering
         \includegraphics[width=\textwidth]{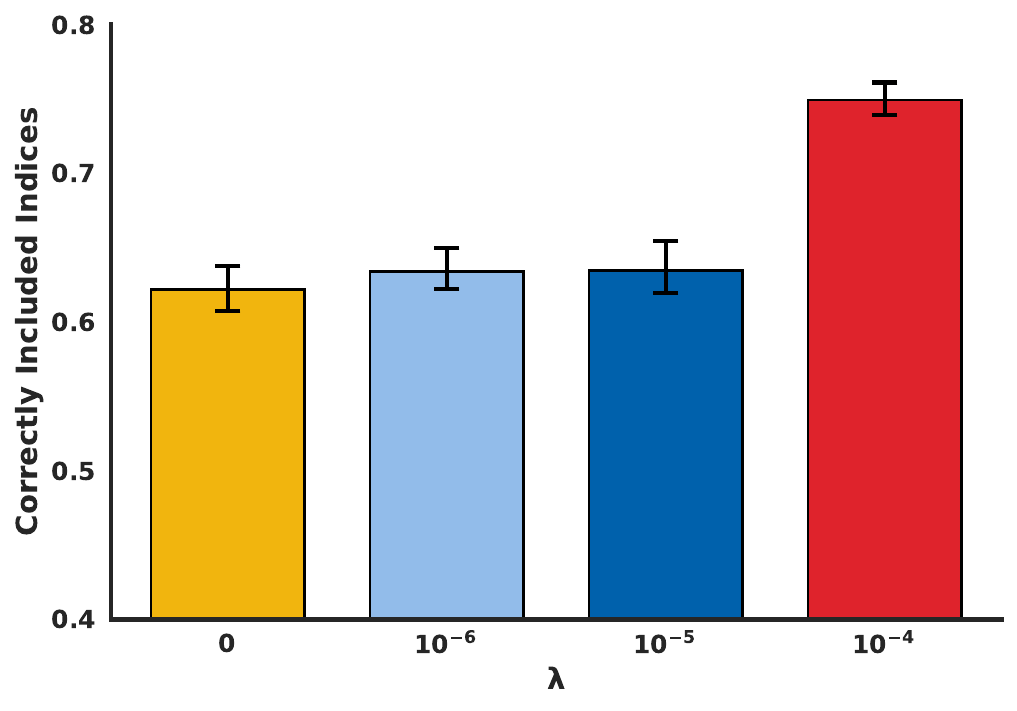}
         \caption{\textbf{Evaluation on inclusion of correct name } We plot the fraction of correctly included first names after intervening on the corresponding feature in the SAE. Error bars are basic bootstrap confidence intervals~\citep{efron1979basicbootstrap} of $100$ random draws.}
         \label{fig:intervenability_include}
     \end{subfigure}
    \caption{\textbf{Intervenability } We intervene on $12$ features in the SAE corresponding to the concept of first names. Then, we evaluate the ability to maintain reasoning performance and the correct generation of the included first name.}
    \label{fig:intervenability}
\end{figure*}

Under the ICM principle, isolated intervention implies that altering one feature does not inform or influence the behavior of other features in the SAE. Translated to natural language modeling, we ideally expect the model to only swap the corresponding feature. To check this, we design the following experiment. Across the different SAEs, we find $12$ features which we interpret as encoding conventionally male first names. We then choose $30$ examples from $\gsmk$ that involve male actors. In each example, we replace the original name with one of the $12$ names and run the standard evaluation task. 

\begin{wrapfigure}{l}{0.5\textwidth}
    \centering
    \begin{tikzpicture}[
    >=stealth,
    shorten >=1pt,
    node distance=1.5cm,
    on grid,
    dimconnection/.style={->, thin, draw=gray!40},
    sparseconnection/.style={->, thin, color=gray!40},
    yellowconnection/.style={->, very thick, color=bauhausyellow},
    blueconnection/.style={->, very thick, color=bauhausblue},
    input/.style={draw=black, circle, minimum size=10pt, fill=bauhausred!20, thick},
    hidden/.style={draw=black, circle, minimum size=10pt, fill=gray!30, thick},
    output/.style={draw=black, circle, minimum size=10pt, fill=bauhausred!20, thick},
    every node/.style={font=\small, align=center}
]
    % --- LAYER 1: Bottom (Input) ---
    \foreach \j in {1,...,5} {
        \node[input] (h1_\j) at ({(\j-3)*1.0cm}, -4.0cm) {};
    }

    % --- LAYER 2: Middle (Hidden) ---
    \foreach \i in {1,...,10} {
        \coordinate (h2_pos) at ({(\i-5.5)*0.75cm}, -2.0cm);
        
        \ifnum\i=3
            \node[hidden, fill=bauhausyellow] (h2_3) at (h2_pos) {};
        \else\ifnum\i=7
            \node[hidden, fill=bauhausblue] (h2_7) at (h2_pos) {};
        \else
            \node[hidden] (h2_\i) at (h2_pos) {};
        \fi\fi
    }

    % --- LAYER 3: Top (Output) ---
    \foreach \j in {1,...,5} {
        \node[output] (h3_\j) at ({(\j-3)*1.0cm}, 0) {};
    }

    % --- CONNECTIONS: Bottom to Top ---
    % \foreach \i in {1,...,10} {
    %     \foreach \j in {1,...,5} {
    %         % From Input (bottom) up to Hidden
    %         \draw[dimconnection] (h1_\j) -- (h2_\i);
    %         % From Hidden up to Output (top)
    %         \draw[dimconnection] (h2_\i) -- (h3_\j);
    %     }
    % }

    % --- SPARSE CONNECTIONS (Upward Flow) ---
    \draw[sparseconnection] (h1_1) -- (h2_1);
    \draw[sparseconnection] (h1_1) -- (h2_8);
    \draw[sparseconnection] (h1_2) -- (h2_1);
    \draw[sparseconnection] (h1_2) -- (h2_3);
    \draw[sparseconnection] (h1_2) -- (h2_8);
    \draw[sparseconnection] (h1_3) -- (h2_8);
    \draw[sparseconnection] (h1_3) -- (h2_1);
    \draw[sparseconnection] (h1_4) -- (h2_1);
    \draw[sparseconnection] (h1_4) -- (h2_3);
    \draw[sparseconnection] (h1_4) -- (h2_8);
    \draw[sparseconnection] (h1_5) -- (h2_1);
    \draw[sparseconnection] (h1_5) -- (h2_3);
    \draw[sparseconnection] (h1_5) -- (h2_8);
    \draw[sparseconnection] (h2_1) -- (h3_1);
    \draw[sparseconnection] (h2_1) -- (h3_2);
    \draw[sparseconnection] (h2_1) -- (h3_3);
    \draw[sparseconnection] (h2_1) -- (h3_4);
    \draw[sparseconnection] (h2_1) -- (h3_5);
    \draw[sparseconnection] (h2_8) -- (h3_1);
    \draw[sparseconnection] (h2_8) -- (h3_2);
    \draw[sparseconnection] (h2_8) -- (h3_3);
    \draw[sparseconnection] (h2_8) -- (h3_4);
    \draw[sparseconnection] (h2_8) -- (h3_5);
    \draw[yellowconnection] (h1_1) -- (h2_3);
    \draw[yellowconnection] (h1_2) -- (h2_3);
    \draw[yellowconnection] (h1_3) -- (h2_3);
    \draw[yellowconnection] (h1_4) -- (h2_3);
    \draw[yellowconnection] (h1_5) -- (h2_3);
    \draw[blueconnection] (h2_7) -- (h3_1);
    \draw[blueconnection] (h2_7) -- (h3_2);
    \draw[blueconnection] (h2_7) -- (h3_3);
    \draw[blueconnection] (h2_7) -- (h3_4);
    \draw[blueconnection] (h2_7) -- (h3_5);
    
    % --- Perpendicular symbols between hidden nodes ---
    % \foreach \i [evaluate=\i as \nexti using int(\i+1)] in {1,...,4} {
    %     \node[draw=none, rectangle, font=\normalsize] 
    %         at ($(h2_\i.east)!0.5!(h2_{\nexti}.west)$) {$\perp$};
    % }
\end{tikzpicture}
\caption{\textbf{Interventions in the SAE } We perform interventions in the SAE by \textcolor{bauhausyellow}{turning off} the intervened on feature and \textcolor{bauhausblue}{adding} to the residual stream the feature we replace with.}
\label{fig:saeinterventions}
\end{wrapfigure}

Having access to the SAE features, however, we are able to swap the feature indices for another one of the $11$ remaining first names. We then check if the model (a) \textit{keeps performance} comparable to~\autoref{fig:mathevaluation}, (b) correctly \textit{drops} the name related to the feature we set to zero, and (c) \textit{includes} the first name we inject into the residual stream through our intervention inside the SAE. Testing all combinations of names on each of the examples, we generate a dataset of $12 * 11 * 30 = 3'960$ examples. \looseness=-1

A well-known issue in mechanistic interpretability, it is non-trivial which value is required to faithfully intervene in the representation space. As the coefficients associated with each feature are continuous, turning the feature on and off is a non-binary problem. In other words, pushing on a feature to be present in the output while keeping generation intact requires testing multiple candidate values~\citep{ameisen2025circuit}. 

\begin{wrapfigure}{l}{0.5\linewidth}    
    \begin{lstlisting}[caption={\textbf{The aqua feature } The top-three input examples activating the $\aqua$ feature at $\lambda = 10^{-4}$ are related to aquaria.},label={lst:aquafive}]
1. found an |u|aquarium| for $10.00
2. in the two |u|aquariums| is twice
3. in the |u|fish tank|. One-third
    \end{lstlisting}
% 1. water of the |u|aqua|rium every week,
% 2. calculate how many |u|aqua|riums she
% 3. 2 tablespoons of |u|aqua|faba is 
% 4. the newly opened |u|aqua|rium and 40
% 5. pass to the |u|aqua|rium for $120
\end{wrapfigure}
We thus sweep over hyperparameters between $10$ and $500$. For any combination of insertion value and orthogonality penalty, the accuracy only moves within an interval of $6$ percentage points. Hence, we decide based on the models' ability to remove the intervened on concept and choose $200$. Confirmation that these findings hold more broadly and the $30$ selected examples are deferred to Appendix~\ref{app:experiments}. % \looseness=-1

\autoref{fig:intervenability} depicts our results. As mentioned, we observe in~\autoref{fig:intervenability_eval} that performance is comparable across different choices for $\lambda$. No setup performs significantly differently from the zero-penalty configuration. Relevant to our hypothesis, for stricter orthogonality penalty, the model significantly more often includes the correct first name in the generation. \autoref{fig:intervenability_include} demonstrates that the orthogonality penalty monotonically increases in $\lambda$. Further, we recover the first name around $74.9\%$ of the time for $\lambda = 10^{-4}$ while only recovering it around $62.2\%$ of the time without penalty. Beyond first names, we extend this analysis to exchanging $\location$ and $\animal$ concepts in the LM output. We also run standard interpretability experiments and obtain constant performance across orthogonality penalties. Further, we run our intervenability experiments on \Llama to strengthen our results. For all this, we refer the reader to Appendix~\ref{app:experiments}. \looseness=-1

To conclude this section, we discuss the insertion of a feature that corresponds to concepts which are not conventionally male first names.~\autoref{lst:aquafive} represents three input examples that activate the same feature. The corresponding explanation in the model with $\lambda = 10^{-4}$ reads\looseness=-1
% \begin{align*}
% \end{align*}
\begin{align*}
    \text{Aquarium capacity or fish population size is the latent concept shared among the spans.}
    % \text{The shared latent concept is the unit "aquarium".}
\end{align*}

Swapping this feature for the feature associated with the first name $\Mike$, we observe that the model adapts this feature to be concordant with the preceding context. In fact, the LM introduces $\Aquaman$. We print the model outputs for this intervention in~\autoref{fig:main}. Moreover,~\autoref{fig:saeinterventions} schematically represents an intervention in the SAE. We note that the feature associated with the concept $\Aquaman$ has no incoming edge while the downstream effect of the feature corresponding to $\Mike$ is turned off.

\subsection{Behavioral and internal measures of interference}
To measure interference beyond accuracy, we report two additional metrics on our intervention experiments. First, we report the $\ROUGE$~\citep{lin2004rouge} between the standard and the intervened generations as an external metric at the token level. This captures how closely the two generations are aligned as a function of the longest common subsequence between generations. In particular, for $S$ and $I$ the standard generation and intervened generation, respectively, we define the longest common subsequence $\LCS$ between $S$ and $I$ as the longest common sequence of tokens between the two. The sequence need not be consecutive. Under an ideal isolated intervention, therefore, only the first name corresponding to the intervened on feature changes while keeping general reasoning intact. For $U_S$ the set of all tokens in our standard generation, $\ROUGE$ is then given by
\begin{align*}
    \ROUGE(S, I) = \frac{\LCS(S, I)}{|U_S|}.
\end{align*}

Second, we compute the symmetric Jensen-Shannon divergence ($\JSD$) between the empirical distributions of the features used in both generations, $P_S$  and $P_I$, as an internal, SAE-level statistic. For $\KL(P || Q)$ the Kullback-Leibler divergence between $P$ and $Q$ and $M = \frac{1}{2} (P_S + P_I)$ a mixture between $P_S$ and $P_I$, the Jensen-Shannon divergence is defined as 
\begin{align*}
    \JSD(P_S || P_I) = \frac{1}{2} \KL(P_S || M) + \frac{1}{2} \KL(P_I || M).
\end{align*}

\autoref{tab:rouge-jsd} displays the two measures for all configurations. Significantly higher $\ROUGE$ and lower $\JSD$ at the strictest orthogonality level $\lambda = 10^{-4}$ relative to the no-penalty setting indicate lower interference. \textit{Cum grano salis}, the results for $\lambda \in \{10^{-6}, 10^{-5}\}$ do not yield definitive conclusions. Studying this further, we present detailed statistics on the additional datasets in Appendix~\ref{app:experiments}. Across datasets, we observe significantly lower $\JSD$ for the strictest penalty compared to $\lambda = 0$ as well as significantly higher $\ROUGE$ on the $\location$ dataset. \looseness=-1

\begin{table}[h!]
\centering
\begin{tabular}{lcc cc}
\toprule
& \multicolumn{2}{c}{$\ROUGE$} & \multicolumn{2}{c}{$\JSD$} \\
\cmidrule(lr){2-3}
\cmidrule(lr){4-5}
$\lambda$ & Mean & Confidence interval & Mean & Confidence interval \\
\midrule
$0$ & 0.763 & [0.756, 0.771] & 0.222 & [0.219, 0.225] \\
$10^{-6}$ & 0.746 & [0.739, 0.750] & 0.254 & [0.252, 0.257] \\
$10^{-5}$ & 0.730 & [0.723, 0.735] & 0.235 & [0.232, 0.239] \\
$10^{-4}$ & 0.773 & [0.766, 0.779] & 0.214 & [0.211, 0.217] \\

\bottomrule
\end{tabular}
\caption{\textbf{Results for interference measures } We evaluate $\ROUGE$ and $\JSD$ on the tokens and features after intervening on the 3'960 examples defined above. We report mean and the basic bootstrap confidence interval over all observations.}
\label{tab:rouge-jsd}
\end{table}
\section{Discussion}
\label{sec:discussion}

Sparse autoencoders (SAEs), which have become a standard tool in mechanistic interpretability, can uncover millions of human-interpretable features in a model’s representation space~\citep{cunningham2023SparseAF,paulo2024automatically}. At the same time, recent work has raised concerns about the reliability of these features. In particular, SAEs trained on the same data with different random seeds often produce substantially different feature sets~\citep{paulo2025sparse}, and features learned by an SAE can themselves be reconstructed from a smaller set of features using a meta-SAE~\citep{leask2025sparse}. Also with respect to feature absorption~\citep{chanin2025absorptionstudyingfeaturesplitting,li2025unlocking}, these results suggest that SAE features may not be truly atomic and may fail to be uniquely identifiable. \looseness=-1

This concern is closely connected to classical results in dictionary learning, where identifiability depends critically on the geometry of the dictionary~\citep{olshausen1997sparse,lee2006efficientsparse}. In particular, recovery guarantees are strongest when the dictionary has low \textit{self-coherence}, meaning that its atoms are nearly orthogonal~\citep{donoho2005stable}. By contrast, when the dictionary is highly self-coherent, multiple distinct sparse combinations of features can generate the same observation~\citep{schnass2008dictionary,garfinkle2016robust}, so that even recovering the sparse coefficients may fail, let alone identifying the underlying features uniquely. Our work suggests that these representational desiderata need not come at the expense of performance: we show that it is possible to preserve fine-tuning performance on a specific task while enforcing activations to be represented as sparse combinations of nearly orthogonal features. Compared to previous work~\citep{korznikov2025ortsaeorthogonalsparseautoencoders}, we enforce global orthogonality across the full dictionary. This points to a possible route toward improving the identifiability and atomicity of SAE features without sacrificing downstream utility. \looseness=-1

Finally, it has recently been hypothesized that adversarial examples are not bugs but a consequence of feature superposition~\citep{gorton2025adversarial}. We view our method as demonstrating a fundamental constraint in this regard: feature interference limits intervention capabilities. When features are represented by non-orthogonal directions, interventions on one feature inevitably spill over to others. Our controlled model-in-the-loop setup makes this limitation explicit and shows that reducing interference improves the reliability of targeted interventions. A natural follow-up question is then whether enforcing greater feature orthogonality reduces superposition and, in turn, mitigates adversarial vulnerabilities. If so, orthogonality could act as a strong architectural bias toward more robust and safer models. Establishing this connection would motivate the development of more scalable versions of our approach. \looseness=-1

\paragraph{Limitations}
Our theoretical study is focused on interventions in the residual stream. Choosing two off-the-shelf SAEs~\citep{karvonen2025saebenchcomprehensivebenchmarksparse,cho2025faithfulsae}, we insert the module in \Llama and \GemmaTwoB only after layers $12$ and $13$, respectively. In consequence, we are constrained to uncovering features that tend to arise in the middle to later layers of the Transformer. Empirically, these concepts often represent high-level concepts relative to the low-level concepts encoded in earlier layers~\citep{dorszewski2025colorsclassesemergenceconcepts}. Additionally, we do not influence the arithmetic problem by changing just the subject of the sentence. In causal terms, we modify a feature that is not an ancestor of the final answer. Thus, the causal mechanism generating the mathematical solution remains unaffected. Extending this with tasks different from mathematical reasoning represents a potential avenue for future research.  \looseness=-1
\section{Conclusion}

We argue that feature entanglement inhibits the ability to perform isolated interventions. Formally relating intervenability to the ICM principle, we provide an upper bound for the expected discrepancy between true and model output in terms of the representation-space interference. We further show that it is possible to fine-tune a LM around a fixed middle layer SAE under orthogonality regularization. While keeping performance on the target dataset unchanged, we reduce the interference between features. In doing so, we perform precise local interventions by exchanging a single concept without affecting others. Moreover, we are able to swap two concepts with greatly improved performance. These results shed light on the practical consequences of feature interference in LMs and suggest that encouraging orthogonality may be a principled way to mitigate it. \looseness=-1

\newpage
\ifcolmsubmission
    \newpage
\fi
\section*{Ethics statement}
This work studies intervenability in language models, with the goal of enabling more precise and localized modifications of internal representations. Improving intervenability can support better control over model behavior and help identify how specific internal components contribute to outputs. These properties are relevant for developing more reliable and predictable AI systems. \looseness=-1

At the same time, increasing the ability to intervene on internal features introduces potential dual-use concerns. More precise control mechanisms could be misused to steer model outputs in unintended or deceptive ways. However, our approach operates at the level of internal representations and requires direct access to model internals and technical expertise. This limits accessibility and reduces the likelihood of misuse in practice. \looseness=-1

Our experiments are conducted on standard benchmarks and do not involve human subjects or sensitive data. The proposed method does not aim to enhance model capabilities, but rather to improve the structure and controllability of learned representations. We believe that advancing intervenability is a useful step toward safer and more transparent AI systems, while recognizing the importance of continued work on safeguards and responsible use. \looseness=-1
\ifcolmpreprint
    \section*{Acknowledgements}

F.D. acknowledges support through a fellowship from the \textit{Hector Fellow Academy}.
\else\ifcolmfinal
    \section*{Acknowledgements}

F.D. acknowledges support through a fellowship from the \textit{Hector Fellow Academy}.
\fi\fi

\bibliography{bibliography,neurips-bibliography}

@misc{wang2025resatransparentreasoningmodels,
      title={Resa: Transparent Reasoning Models via SAEs}, 
      author={Shangshang Wang and Julian Asilis and Ömer Faruk Akgül and Enes Burak Bilgin and Ollie Liu and Deqing Fu and Willie Neiswanger},
      year={2025},
      eprint={2506.09967},
      archivePrefix={arXiv},
      primaryClass={cs.CL},
      url={https://arxiv.org/abs/2506.09967}, 
}

@article{gorton2025adversarial,
  title={Adversarial Examples Are Not Bugs, They Are Superposition},
  author={Gorton, Liv and Lewis, Owen},
  journal={arXiv preprint arXiv:2508.17456},
  year={2025},
  url={https://arxiv.org/abs/2508.17456}
}

@InProceedings{checn2025lowrankadapting,
  title = 	 {Low-Rank Adapting Models for Sparse Autoencoders},
  author =       {Chen, Matthew and Engels, Joshua and Tegmark, Max},
  booktitle = 	 {Proceedings of the 42nd International Conference on Machine Learning},
  pages = 	 {8077--8092},
  year = 	 {2025},
  editor = 	 {Singh, Aarti and Fazel, Maryam and Hsu, Daniel and Lacoste-Julien, Simon and Berkenkamp, Felix and Maharaj, Tegan and Wagstaff, Kiri and Zhu, Jerry},
  volume = 	 {267},
  series = 	 {Proceedings of Machine Learning Research},
  month = 	 {13--19 Jul},
  publisher =    {PMLR},
  url = 	 {https://proceedings.mlr.press/v267/chen25r.html}
}

@inproceedings{lee2006efficientsparse,
	author = {Lee, Honglak and Battle, Alexis and Raina, Rajat and Ng, Andrew},
	booktitle = {Advances in Neural Information Processing Systems},
	editor = {B. Sch\"{o}lkopf and J. Platt and T. Hoffman},
	publisher = {MIT Press},
	title = {Efficient sparse coding algorithms},
	url = {https://proceedings.neurips.cc/paper_files/paper/2006/file/2d71b2ae158c7c5912cc0bbde2bb9d95-Paper.pdf},
	volume = {19},
	year = {2006}}

@misc{karvonen2025saebenchcomprehensivebenchmarksparse,
      title={SAEBench: A Comprehensive Benchmark for Sparse Autoencoders in Language Model Interpretability}, 
      author={Adam Karvonen and Can Rager and Johnny Lin and Curt Tigges and Joseph Bloom and David Chanin and Yeu-Tong Lau and Eoin Farrell and Callum McDougall and Kola Ayonrinde and Demian Till and Matthew Wearden and Arthur Conmy and Samuel Marks and Neel Nanda},
      year={2025},
      eprint={2503.09532},
      archivePrefix={arXiv},
      primaryClass={cs.LG},
      url={https://arxiv.org/abs/2503.09532}, 
}

@inproceedings{
    gao2025topk,
    title={Scaling and evaluating sparse autoencoders},
    author={Leo Gao and Tom Dupre la Tour and Henk Tillman and Gabriel Goh and Rajan Troll and Alec Radford and Ilya Sutskever and Jan Leike and Jeffrey Wu},
    booktitle={The Thirteenth International Conference on Learning Representations},
    year={2025},
    url={https://openreview.net/forum?id=tcsZt9ZNKD}
}

@article{makhzani2013ksae,
  title={K-sparse autoencoders},
  author={Makhzani, Alireza and Frey, Brendan},
  journal={arXiv preprint arXiv:1312.5663},
  year={2013}
}

@inproceedings{jiang2024ontheorigins,
author = {Jiang, Yibo and Rajendran, Goutham and Ravikumar, Pradeep and Aragam, Bryon and Veitch, Victor},
title = {On the origins of linear representations in large language models},
year = {2024},
publisher = {JMLR.org},
booktitle = {Proceedings of the 41st International Conference on Machine Learning},
articleno = {879},
numpages = {33},
location = {Vienna, Austria},
series = {ICML'24}
}

@article{marks2024enhancing,
  title={Enhancing neural network interpretability with feature-aligned sparse autoencoders},
  author={Marks, Luke and Paren, Alasdair and Krueger, David and Barez, Fazl},
  journal={arXiv preprint arXiv:2411.01220},
  year={2024}
}

@article{leask2025sparse,
  title={Sparse autoencoders do not find canonical units of analysis},
  author={Leask, Patrick and Bussmann, Bart and Pearce, Michael and Bloom, Joseph and Tigges, Curt and Moubayed, Noura Al and Sharkey, Lee and Nanda, Neel},
  journal={arXiv preprint arXiv:2502.04878},
  year={2025}
}

@article{paulo2024automatically,
  title={Automatically interpreting millions of features in large language models},
  author={Paulo, Gon{\c{c}}alo and Mallen, Alex and Juang, Caden and Belrose, Nora},
  journal={arXiv preprint arXiv:2410.13928},
  year={2024}
}

@article{paulo2025sparse,
  title={Sparse autoencoders trained on the same data learn different features},
  author={Paulo, Gon{\c{c}}alo and Belrose, Nora},
  journal={arXiv preprint arXiv:2501.16615},
  year={2025}
}

@article{donoho2005stable,
  title={Stable recovery of sparse overcomplete representations in the presence of noise},
  author={Donoho, David L and Elad, Michael and Temlyakov, Vladimir N},
  journal={IEEE Transactions on information theory},
  volume={52},
  number={1},
  pages={6--18},
  year={2005},
  publisher={IEEE}
}

@article{schnass2008dictionary,
  title={Dictionary preconditioning for greedy algorithms},
  author={Schnass, Karin and Vandergheynst, Pierre},
  journal={IEEE Transactions on Signal Processing},
  volume={56},
  number={5},
  pages={1994--2002},
  year={2008},
  publisher={IEEE}
}

@article{garfinkle2016robust,
  title={Robust Identifiability in Sparse Dictionary Learning},
  author={Garfinkle, Charles J and Hillar, Christopher J},
  journal={arXiv preprint arXiv:1606.06997},
  year={2016}
}

@ARTICLE{welch1974lowerbounds,
  author={Welch, L.},
  journal={IEEE Transactions on Information Theory}, 
  title={Lower bounds on the maximum cross correlation of signals (Corresp.)}, 
  year={1974},
  volume={20},
  number={3},
  pages={397-399},
  doi={10.1109/TIT.1974.1055219}}

@ARTICLE{waldron2003generalizedwelch,
  author={Waldron, S.},
  journal={IEEE Transactions on Information Theory}, 
  title={Generalized Welch bound equality sequences are tight frames}, 
  year={2003},
  volume={49},
  number={9},
  pages={2307-2309},
  doi={10.1109/TIT.2003.815788}}

@article{yu2023metamath,
  title={MetaMath: Bootstrap Your Own Mathematical Questions for Large Language Models},
  author={Yu, Longhui and Jiang, Weisen and Shi, Han and Yu, Jincheng and Liu, Zhengying and Zhang, Yu and Kwok, James T and Li, Zhenguo and Weller, Adrian and Liu, Weiyang},
  journal={arXiv preprint arXiv:2309.12284},
  year={2023}
}

@article{cobbe2021gsm8k,
  title={Training Verifiers to Solve Math Word Problems},
  author={Cobbe, Karl and Kosaraju, Vineet and Bavarian, Mohammad and Chen, Mark and Jun, Heewoo and Kaiser, Lukasz and Plappert, Matthias and Tworek, Jerry and Hilton, Jacob and Nakano, Reiichiro and Hesse, Christopher and Schulman, John},
  journal={arXiv preprint arXiv:2110.14168},
  year={2021}
}

@book{Spirtes2000causation,
  author = {Spirtes, P. and Glymour, C. and Scheines, R.},
  edition = {2nd},
  publisher = {MIT press},
  review = {PC algorithm},
  title = {Causation, Prediction, and Search},
  year = 2000
}

@misc{mencattini2026exploratorycausalinferencesaence,
      title={Exploratory Causal Inference in SAEnce}, 
      author={Tommaso Mencattini and Riccardo Cadei and Francesco Locatello},
      year={2026},
      eprint={2510.14073},
      archivePrefix={arXiv},
      primaryClass={cs.LG},
      url={https://arxiv.org/abs/2510.14073}, 
}

@article{HYVARINEN1999ica,
	author = {Aapo Hyv{\"a}rinen and Petteri Pajunen},
	doi = {https://doi.org/10.1016/S0893-6080(98)00140-3},
	issn = {0893-6080},
	journal = {Neural Networks},
	number = {3},
	pages = {429-439},
	title = {Nonlinear independent component analysis: Existence and uniqueness results},
	url = {https://www.sciencedirect.com/science/article/pii/S0893608098001403},
	volume = {12},
	year = {1999}}

@InProceedings{locatello2019challenging,
  title = 	 {Challenging Common Assumptions in the Unsupervised Learning of Disentangled Representations},
  author =       {Locatello, Francesco and Bauer, Stefan and Lucic, Mario and Raetsch, Gunnar and Gelly, Sylvain and Sch{\"o}lkopf, Bernhard and Bachem, Olivier},
  booktitle = 	 {Proceedings of the 36th International Conference on Machine Learning},
  pages = 	 {4114--4124},
  year = 	 {2019},
  editor = 	 {Chaudhuri, Kamalika and Salakhutdinov, Ruslan},
  volume = 	 {97},
  series = 	 {Proceedings of Machine Learning Research},
  month = 	 {09--15 Jun},
  publisher =    {PMLR},
  url = 	 {https://proceedings.mlr.press/v97/locatello19a.html}
}

@inproceedings{buchholz2023learninglinear,
author = {Buchholz, Simon and Rajendran, Goutham and Rosenfeld, Elan and Aragam, Bryon and Sch\"{o}lkopf, Bernhard and Ravikumar, Pradeep},
title = {Learning linear causal representations from interventions under general nonlinear mixing},
year = {2023},
publisher = {Curran Associates Inc.},
address = {Red Hook, NY, USA},
booktitle = {Proceedings of the 37th International Conference on Neural Information Processing Systems},
articleno = {1968},
numpages = {44},
location = {New Orleans, LA, USA},
series = {NIPS '23}
}

@article{gershgorin1931circle,
  author = {Gershgorin, S.},
  journal = {Izvestija Akademii Nauk SSSR, Serija Matematika},
  number = 3,
  pages = {749--754},
  title = {Uber die Abgrenzung der Eigenwerte einer Matrix},
  volume = 7,
  year = 1931
}

@article{ameisen2025circuit,
  author={Ameisen, Emmanuel and Lindsey, Jack and Pearce, Adam and Gurnee, Wes and Turner, Nicholas L. and Chen, Brian and Citro, Craig and Abrahams, David and Carter, Shan and Hosmer, Basil and Marcus, Jonathan and Sklar, Michael and Templeton, Adly and Bricken, Trenton and McDougall, Callum and Cunningham, Hoagy and Henighan, Thomas and Jermyn, Adam and Jones, Andy and Persic, Andrew and Qi, Zhenyi and Ben Thompson, T. and Zimmerman, Sam and Rivoire, Kelley and Conerly, Thomas and Olah, Chris and Batson, Joshua},
  title={Circuit Tracing: Revealing Computational Graphs in Language Models},
  journal={Transformer Circuits Thread},
  year={2025},
  url={https://transformer-circuits.pub/2025/attribution-graphs/methods.html}
}

@misc{dorszewski2025colorsclassesemergenceconcepts,
      title={From Colors to Classes: Emergence of Concepts in Vision Transformers}, 
      author={Teresa Dorszewski and Lenka Tětková and Robert Jenssen and Lars Kai Hansen and Kristoffer Knutsen Wickstrøm},
      year={2025},
      eprint={2503.24071},
      archivePrefix={arXiv},
      primaryClass={cs.CV},
      url={https://arxiv.org/abs/2503.24071}, 
}

@article{gao2025weight,
  title={Weight-sparse transformers have interpretable circuits},
  author={Gao, Leo and Rajaram, Achyuta and Coxon, Jacob and Govande, Soham V and Baker, Bowen and Mossing, Dan},
  journal={arXiv preprint arXiv:2511.13653},
  year={2025}
}

@article{draye2025sparse,
  title={Sparse Attention Post-Training for Mechanistic Interpretability},
  author={Draye, Florent and Lei, Anson and Posner, Ingmar and Sch{\"o}lkopf, Bernhard},
  journal={arXiv preprint arXiv:2512.05865},
  year={2025}
}

@article{chen2026emergence,
  title={On the Emergence and Test-Time Use of Structural Information in Large Language Models},
  author={Chen, Michelle Chao and Miller, Moritz and Sch{\"o}lkopf, Bernhard and Guo, Siyuan},
  journal={arXiv preprint arXiv:2601.17869},
  year={2026}
}

@inproceedings{lin2004rouge,
    title = "{ROUGE}: A Package for Automatic Evaluation of Summaries",
    author = "Lin, Chin-Yew",
    booktitle = "Text Summarization Branches Out",
    month = jul,
    year = "2004",
    address = "Barcelona, Spain",
    publisher = "Association for Computational Linguistics",
    url = "https://aclanthology.org/W04-1013/",
    pages = "74--81"
}

@inproceedings{cho2025faithfulsae,
  title={Faithful{SAE}: Towards Capturing Faithful Features with Sparse Autoencoders without External Datasets Dependency},
  author={Seonglae Cho and Harryn Oh and Donghyun Lee and Luis Eduardo Rodrigues Vieira and Andrew Bermingham and Ziad El Sayed},
  booktitle={ACL 2025 Student Research Workshop},
  year={2025},
  url={https://openreview.net/forum?id=tBn9ChHGG9}
}

@inproceedings{li2025unlocking,
    title={Unlocking Hierarchical Concept Discovery in Language Models Through Geometric Regularization},
    author={Ed Li and Junyu Ren},
    booktitle={ICLR 2025 Workshop on Building Trust in Language Models and Applications},
    year={2025},
    url={https://openreview.net/forum?id=i31cKXiyim}
}

@misc{korznikov2025ortsaeorthogonalsparseautoencoders,
      title={OrtSAE: Orthogonal Sparse Autoencoders Uncover Atomic Features}, 
      author={Anton Korznikov and Andrey Galichin and Alexey Dontsov and Oleg Rogov and Elena Tutubalina and Ivan Oseledets},
      year={2025},
      eprint={2509.22033},
      archivePrefix={arXiv},
      primaryClass={cs.LG},
      url={https://arxiv.org/abs/2509.22033}, 
}

@misc{chanin2025absorptionstudyingfeaturesplitting,
      title={A is for Absorption: Studying Feature Splitting and Absorption in Sparse Autoencoders}, 
      author={David Chanin and James Wilken-Smith and Tomáš Dulka and Hardik Bhatnagar and Satvik Golechha and Joseph Bloom},
      year={2025},
      eprint={2409.14507},
      archivePrefix={arXiv},
      primaryClass={cs.CL},
      url={https://arxiv.org/abs/2409.14507}, 
}

@misc{dofinetti,
      title={Do Finetti: On Causal Effects for Exchangeable Data}, 
      author={Siyuan Guo and Chi Zhang and Karthika Mohan and Ferenc Huszár and Bernhard Schölkopf},
      year={2024},
      eprint={2405.18836},
      archivePrefix={arXiv},
      primaryClass={stat.ME},
      url={https://arxiv.org/abs/2405.18836}, 
}

@misc{causaldefinetti,
      title={Causal de Finetti: On the Identification of Invariant Causal Structure in Exchangeable Data}, 
      author={Siyuan Guo and Viktor Tóth and Bernhard Schölkopf and Ferenc Huszár},
      year={2024},
      eprint={2203.15756},
      archivePrefix={arXiv},
      primaryClass={stat.ML},
      url={https://arxiv.org/abs/2203.15756}, 
}

@book{elements, author = {Peters, Jonas and Janzing, Dominik and Schölkopf, Bernhard}, title = {Elements of Causal Inference: Foundations and Learning Algorithms}, year = {2017}, isbn = {0262037319}, publisher = {The MIT Press} }

@book{pearl2009,
  author = {Pearl, Judea},
  edition = {2nd},
  publisher = {Cambridge University Press},
  title = {Causality: Models, Reasoning and Inference},
  year = 2009,
  url = {https://bayes.cs.ucla.edu/BOOK-2K/}
}

@article{icp,
	author = {Peters, Jonas and B{\"u}hlmann, Peter and Meinshausen, Nicolai},
	doi = {10.1111/rssb.12167},
	eprint = {https://academic.oup.com/jrsssb/article-pdf/78/5/947/49235444/jrsssb\_78\_5\_947.pdf},
	issn = {1369-7412},
	journal = {Journal of the Royal Statistical Society Series B: Statistical Methodology},
	month = {10},
	number = {5},
	pages = {947-1012},
	title = {Causal Inference by using Invariant Prediction: Identification and Confidence Intervals},
	url = {https://doi.org/10.1111/rssb.12167},
	volume = {78},
	year = {2016}}

@article{Janzing2010CausalCondition,
	author = {Janzing, D and Schölkopf, B},
	journal = {IEEE Transactions on Information Theory},
	number = {10},
	title = {{Causal inference using the algorithmic Markov condition}},
	volume = {56},
	year = {2010}}

@article{efron1979basicbootstrap,
 ISSN = {00905364, 21688966},
 URL = {http://www.jstor.org/stable/2958830},
 author = {B. Efron},
 journal = {The Annals of Statistics},
 number = {1},
 pages = {1--26},
 publisher = {Institute of Mathematical Statistics},
 title = {Bootstrap Methods: Another Look at the Jackknife},
 urldate = {2025-05-11},
 volume = {7},
 year = {1979}
}

@inproceedings{park2024lrh,
author = {Park, Kiho and Choe, Yo Joong and Veitch, Victor},
title = {The linear representation hypothesis and the geometry of large language models},
year = {2024},
publisher = {JMLR.org},
booktitle = {Proceedings of the 41st International Conference on Machine Learning},
articleno = {1605},
numpages = {24},
location = {Vienna, Austria},
series = {ICML'24}
}

@misc{iccr,
      title={Counterfactual reasoning: an analysis of in-context emergence}, 
      author={Moritz Miller and Bernhard Schölkopf and Siyuan Guo},
      year={2025},
      eprint={2506.05188},
      archivePrefix={arXiv},
      primaryClass={cs.CL},
      url={https://arxiv.org/abs/2506.05188}, 
}

@article{bernhard2021towardcrl,
  title = {Toward Causal Representation Learning},
  journal = {Proceedings of the IEEE},
  volume = {109},
  number = {5},
  pages = {612--634},
  year = {2021},
  note = {*equal contribution},
  author = {Sch{\"o}lkopf, B. and Locatello, F. and Bauer, S. and Ke, N. R. and Kalchbrenner, N. and Goyal, A. and Bengio, Y.},
  doi = {10.1109/JPROC.2021.3058954},
  url = {https://ieeexplore.ieee.org/stamp/stamp.jsp?arnumber=9363924}
}

@article{cunningham2023SparseAF,
  title={Sparse Autoencoders Find Highly Interpretable Features in Language Models},
  author={Hoagy Cunningham and Aidan Ewart and Logan Riggs Smith and Robert Huben and Lee Sharkey},
  journal={ArXiv},
  year={2023},
  volume={abs/2309.08600},
  url={https://api.semanticscholar.org/CorpusID:261934663}
}

@inproceedings{Geiger2023CausalAA,
  title={Causal Abstraction: A Theoretical Foundation for Mechanistic Interpretability},
  author={Atticus Geiger and Duligur Ibeling and Amir Zur and Maheep Chaudhary and Sonakshi Chauhan and Jing Huang and Aryaman Arora and Zhengxuan Wu and Noah D. Goodman and Christopher Potts and Thomas F. Icard},
  year={2023},
  url={https://api.semanticscholar.org/CorpusID:255749463}
}

@article{hu2021lora,
  title={LoRA: Low-Rank Adaptation of Large Language Models},
  author={J. Edward Hu and Yelong Shen and Phillip Wallis and Zeyuan Allen-Zhu and Yuanzhi Li and Shean Wang and Weizhu Chen},
  journal={ArXiv},
  year={2021},
  volume={abs/2106.09685},
  url={https://api.semanticscholar.org/CorpusID:235458009}
}

@article{elhage2022superposition,
   title={Toy Models of Superposition},
   author={Elhage, Nelson and Hume, Tristan and Olsson, Catherine and Schiefer, Nicholas and Henighan, Tom and Kravec, Shauna and Hatfield-Dodds, Zac and Lasenby, Robert and Drain, Dawn and Chen, Carol and Grosse, Roger and McCandlish, Sam and Kaplan, Jared and Amodei, Dario and Wattenberg, Martin and Olah, Christopher},
   year={2022},
   journal={Transformer Circuits Thread},
   url={https://transformer-circuits.pub/2022/toy_model/index.html}
}

@article{olshausen1997sparse,
	author = {Bruno A. Olshausen and David J. Field},
	doi = {https://doi.org/10.1016/S0042-6989(97)00169-7},
	issn = {0042-6989},
	journal = {Vision Research},
	number = {23},
	pages = {3311-3325},
	title = {Sparse coding with an overcomplete basis set: A strategy employed by V1?},
	url = {https://www.sciencedirect.com/science/article/pii/S0042698997001697},
	volume = {37},
	year = {1997}}
\bibliographystyle{colm/colm2026_conference}

\appendix
\section{Self-coherence and Welch bounds}
\label{app:welch}

\subsection{Self-coherence}
\label{sec:app-self-coherence}
To illustrate the relevance of self-coherence for sparsity, we restate the below theorem~\citep{donoho2005stable} and provide the proof. Theorem~\ref{thm:self-coherence} reflects a clear relationship between sparsity and self-coherence. Note that we define self-coherence here as the maximum interference between features, $\mu_\text{max}$. Below, we discuss how they relate to~\eqref{eqn:self-coherence} through the Welch bounds. First, low self-coherence corresponds to stricter almost orthogonality between dictionary atoms. Second, this yields a bound on the number of atoms required for input recovery $\K$ to ensure identifiability. \looseness=-1

\begin{theorem}[Self-coherence bound for uniqueness] 
Let $\mD \in \Rmd$ have unit-norm columns and denote by $\mu_\text{max}$  self-coherence as defined in~\eqref{eqn:self-coherence}. If
\begin{align*}
    \K < \frac{1}{2} \left(1+\frac{1}{\mu_\text{max}}\right),
\end{align*}
every $\K$-sparse representation $\vxtilde = \mD \vz$ is unique. That is,
for $\mD \vz = \mD \tilde{\vz}$ with $\|\vz\|_0, \|\tilde{\vz}\|_0 \leq \K$, it holds that $\vz = \tilde{\vz}$.
\label{thm:self-coherence}
\end{theorem}

\begin{proof}
Assume $\vxtilde = \mD \vz = \mD \tilde{\vz}$ with $\vz \neq \tilde{\vz}$, both $\K$-sparse.  
Then $\vs = \vz - \tilde{\vz} \neq \zeros$ satisfies $\mD\vs = \zeros$. Let T = $\mathrm{supp}(\vs)$, then $|T| \leq 2K$. Let $\mG = \mD_T^\transpose \mD_T$.
All diagonal elements of $\mG$ are equal to $1$ with off-diagonal elements upper-bounded by $\mu_\text{max}$. Formally, $\mG_{j,j} = 1, \forall j \in [d]$ and $\mG_{i,j} \leq \mu_\text{max}$ for $i \neq j$ with $i, j \in [d]$. By Gershgorin's circle theorem~\citep{gershgorin1931circle}, \looseness=-1
\begin{align*}
    \lambda_{\min}(\mG) \geq 1 - (2\K - 1)\mu_\text{max} > 0,
\end{align*}
by assumption on $\K$. Thus $\mG$ is positive definite and the atoms in $T$ are independent. We arrive at a contradiction.
\end{proof}

\subsection{Welch bounds}
% We symbolize unit-norm vectors as $\normvec{\vv}$ for arbitrary vector $\vv$. 

For $\{\normvec{\vf_j}\}_{j \in [d]}$ unit-norm vectors with $\normvec{\vf_j} \in \Rmodel$, define the maximum similarity by $ \normvec{\mu_\text{max}} = \max_{j \neq k} |\langle \normvec{\vf_j}, \normvec{\vf_k} \rangle|$. Then, for $s \in \N_{>0}$,
\begin{align*}
    \normvec{\mu_\text{max}}^{2s} \geq \frac{1}{d-1}\left[ \frac{d}{\binom{m+s-1}{s}} - 1 \right],
\end{align*}
with $\N_{>0}$ denoting the non-zero natural numbers. In addition, Welch obtains a bound on the mean similarity, which we state in~\eqref{eqn:generalizedwelch},
\begin{align*}
    \frac{1}{d^2} \sum_{j = 1}^d \sum_{k = 1}^d |\langle \normvec{\vf_j}, \normvec{\vf_k} \rangle|^{2s} \geq \frac{1}{\binom{m+s-1}{s}}
\end{align*}
for which~\citet{waldron2003generalizedwelch} relaxes the unit-norm assumption,
\begin{align*}
        \frac{\sum_{j = 1}^d \sum_{k = 1}^d |\langle \vf_j, \vf_k \rangle|^{2s}}{\left(\sum_{j = 1}^d \|\vf_j\|\right)^{2s}} \geq \frac{1}{\binom{m+s-1}{s}}.
\end{align*}
with $\vf_j$ and $\vf_k$ arbitrary.
\section{Post-intervention interference between features}
\label{app:interference}

% \subsection{Orthogonality penalty equivalence}
% \moritz{write}$\|\mI - \mD^\transpose \mD\|^2$

\subsection{Causal intervention mismatch between the true output and the model output}

\begin{theorem}[Upper bound on interference]
Let $\mD=\AtomsBracket \in \Rmd$ and assume that
\[
\|\vf_j\|_2=1,\qquad j \in [d].
\]
Fix $j \in [d]$ and \(\delta\in\mathbb{R}\). Then
\[
\mathbb{E}\Bigl[\|Y^{(j,\delta)}-\widehat Y^{(j,\delta)}\|_2^2\Bigr]
\le
L^2\|\mA\|_{\mathrm{op}}^2\,
\mathbb{E}\Bigl[\|(\mG-\mI)(Z+\delta \ve_j)\|_2^2\Bigr].
\]
where $L$ is a bound on the operator norm of the Jacobian $\veta\mapsto p_\veta$. Consequently,
\[
\mathbb{E}\Bigl[\|Y^{(j,\delta)}-\widehat Y^{(j,\delta)}\|_2^2\Bigr]
\le
2L^2\|\mA\|_{\mathrm{op}}^2
\Bigl(
\mathbb{E}[\|(\mG-\mI)Z\|_2^2]
+
\delta^2\sum_{k\neq j}\langle \vf_k,\vf_j\rangle^2
\Bigr).
\]
For $\mu(\mD)$ as defined in~\eqref{eqn:self-coherence}, averaging over \(j\) yields
\[
\frac1d\sum_{j=1}^d
\mathbb{E}\Bigl[\|Y^{(j,\delta)}-\widehat Y^{(j,\delta)}\|_2^2\Bigr]
\le
2L^2\|\mA\|_{\mathrm{op}}^2\,
\mathbb{E}[\|(\mG-\mI)Z\|_2^2]
+
2L^2\|\mA\|_{\mathrm{op}}^2\,\delta^2(d-1)\mu(\mD),
\]
% \[
% \mu(\mD):=\frac{1}{d(d-1)}\sum_{j\neq k}\langle \vf_j,\vf_k\rangle^2.
% \]
In the standard softmax case, in particular, one may take \(L=\frac12\) such that
\[
\mathbb{E}\Bigl[\|Y^{(j,\delta)}-\widehat Y^{(j,\delta)}\|_2^2\Bigr]
\le
\frac14\|\mA\|_{\mathrm{op}}^2\,
\mathbb{E}\Bigl[\|(\mG-\mI)(Z+\delta \ve_j)\|_2^2\Bigr].
\]
\end{theorem}

% Let \(D=[f_1,\dots,f_d]\in\mathbb{R}^{m\times d}\) and assume that
% \[
% \|f_j\|_2=1,\qquad j=1,\dots,d.
% \]
% Fix \(j\in\{1,\dots,d\}\) and \(\delta\in\mathbb{R}\). Then
% \[
% \mathbb{E}\Bigl[\|Y^{(j,\delta)}-\widehat Y^{(j,\delta)}\|_2^2\Bigr]
% \le
% L^2\|A\|_{\mathrm{op}}^2\,
% \mathbb{E}\Bigl[\|(G-I)(Z+\delta e_j)\|_2^2\Bigr].
% \]
% Consequently,
% \[
% \mathbb{E}\Bigl[\|Y^{(j,\delta)}-\widehat Y^{(j,\delta)}\|_2^2\Bigr]
% \le
% 2L^2\|A\|_{\mathrm{op}}^2
% \Bigl(
% \mathbb{E}[\|(G-I)Z\|_2^2]
% +
% \delta^2\sum_{k\neq j}\langle f_k,f_j\rangle^2
% \Bigr).
% \]
% Averaging over \(j\) yields
% \[
% \frac1d\sum_{j=1}^d
% \mathbb{E}\Bigl[\|Y^{(j,\delta)}-\widehat Y^{(j,\delta)}\|_2^2\Bigr]
% \le
% 2L^2\|A\|_{\mathrm{op}}^2\,
% \mathbb{E}[\|(G-I)Z\|_2^2]
% +
% 2L^2\|A\|_{\mathrm{op}}^2\,\delta^2(d-1)\mu(\mD),
% \]
% where
% \[
% \mu(\mD):=\frac{1}{d(d-1)}\sum_{j\neq k}\langle f_j,f_k\rangle^2.
% \]
% In particular, in the standard softmax case one may take \(L=\frac12\), giving
% \[
% \mathbb{E}\Bigl[\|Y^{(j,\delta)}-\widehat Y^{(j,\delta)}\|_2^2\Bigr]
% \le
% \frac14\|A\|_{\mathrm{op}}^2\,
% \mathbb{E}\Bigl[\|(G-I)(Z+\delta e_j)\|_2^2\Bigr].
% \]
% \end{theorem}

\begin{proof}
Under the intervention \(\doop(Z_j:=Z_j+\delta)\), the latent state becomes
\[
Z^{(j,\delta)}=Z+\delta \ve_j.
\]
The true post-intervention output distribution is therefore parameterized by
\[
\veta_Y^{(j,\delta)}:=\mA(Z+\delta \ve_j),
\]
while the model post-intervention output distribution is parameterized by
\[
\veta_{\widehat Y}^{(j,\delta)}:=\mA \mG(Z+\delta \ve_j).
\]
Subtracting gives
\[
\veta_{\widehat Y}^{(j,\delta)}-\veta_Y^{(j,\delta)}
=
\mA(\mG - \mI)(Z+\delta \ve_j).
\]
Hence,
\[
\|\veta_{\widehat Y}^{(j,\delta)}-\veta_Y^{(j,\delta)}\|_2
\le
\|\mA\|_{\mathrm{op}}\,
\|(\mG - \mI)(Z+\delta \ve_j)\|_2.
\]

We now pass from parameters to output distributions. Since the Jacobian of \(\veta\mapsto p_\veta\) is bounded by \(L\), the mean value theorem implies
\[
\|p_\veta-p_{\veta'}\|_2
\le
L\|\veta-\veta'\|_2
\qquad\text{for all }\veta,\veta'\in\mathbb{R}^q.
\]
Applying this with
\[
\veta=\veta_Y^{(j,\delta)},
\qquad
\veta'=\veta_{\widehat Y}^{(j,\delta)},
\]
yields
\[
\|Y^{(j,\delta)}-\widehat Y^{(j,\delta)}\|_2
\le
L\,\|\veta_{\widehat Y}^{(j,\delta)}-\veta_Y^{(j,\delta)}\|_2.
\]
Combining the two inequalities, we obtain
\[
\|Y^{(j,\delta)}-\widehat Y^{(j,\delta)}\|_2
\le
L\|\mA\|_{\mathrm{op}}\,
\|(\mG - \mI)(Z+\delta \ve_j)\|_2.
\]
Squaring both sides and taking expectation proves the first bound,
\[
\mathbb{E}\Bigl[\|Y^{(j,\delta)}-\widehat Y^{(j,\delta)}\|_2^2\Bigr]
\le
L^2\|\mA\|_{\mathrm{op}}^2\,
\mathbb{E}\Bigl[\|(\mG - \mI)(Z+\delta \ve_j)\|_2^2\Bigr].
\]

To obtain the more explicit estimate, expand
\[
(\mG - \mI)(Z+\delta \ve_j)=(\mG - \mI)Z+\delta(\mG - \mI)\ve_j.
\]
Using
\[
\|\vu+\vv\|_2^2\le 2\|\vu\|_2^2+2\|\vv\|_2^2,
\]
we get
\[
\|(\mG - \mI)(Z+\delta \ve_j)\|_2^2
\le
2\|(\mG - \mI)Z\|_2^2+2\delta^2\|(\mG - \mI)\ve_j\|_2^2.
\]
Substituting into the previous inequality gives
\[
\mathbb{E}\Bigl[\|Y^{(j,\delta)}-\widehat Y^{(j,\delta)}\|_2^2\Bigr]
\le
2L^2\|\mA\|_{\mathrm{op}}^2
\Bigl(
\mathbb{E}[\|(\mG - \mI)Z\|_2^2]
+
\delta^2\|(\mG - \mI)\ve_j\|_2^2
\Bigr).
\]

It remains to compute \(\|(\mG - \mI)\ve_j\|_2^2\). Since \(\mG=\mD^\top \mD\), we have
\[
\mG_{kj}=\langle \vf_k,\vf_j\rangle.
\]
Because the columns are unit norm, \(\mG_{jj}=1\). Thus
\[
((\mG - \mI)\ve_j)_k=
\begin{cases}
0,&k=j,\\[3pt]
\langle \vf_k,\vf_j\rangle,&k\neq j,
\end{cases}
\]
and therefore
\[
\|(\mG - \mI)\ve_j\|_2^2
=
\sum_{k\neq j}\langle \vf_k,\vf_j\rangle^2.
\]
Substituting this identity proves the second bound.

Finally, averaging over \(j\), we obtain
\[
\frac1d\sum_{j=1}^d
\mathbb{E}\Bigl[\|Y^{(j,\delta)}-\widehat Y^{(j,\delta)}\|_2^2\Bigr]
\le
2L^2\|\mA\|_{\mathrm{op}}^2\,
\mathbb{E}[\|(\mG - \mI)Z\|_2^2]
+
2L^2\|\mA\|_{\mathrm{op}}^2\,\delta^2
\frac1d\sum_{j=1}^d\sum_{k\neq j}\langle \vf_k,\vf_j\rangle^2.
\]
By the definition of \(\mu(\mD)\) in~\eqref{eqn:self-coherence},
\[
\sum_{j=1}^d\sum_{k\neq j}\langle \vf_k,\vf_j\rangle^2
=
d(d-1)\mu(\mD),
\]
we obtain the averaged inequality.
\end{proof}

\subsection{\(\K\)-sparse latent states.}
% The theorem shows that the post-intervention mismatch is controlled by a baseline term,
% \[
% \mathbb E[\|(\mG - \mI)Z\|_2^2],
% \]
% and an intervention-specific term,
% \[
% \delta^2\sum_{k\neq j}\langle f_k,f_j\rangle^2.
% \]
% Thus, when the columns of \(D\) are nearly orthogonal, interventions on a latent feature induce only limited distortion in the model output. The constant \(L\) isolates the sensitivity of the output distribution to perturbations of its parameterization, with \(L=\frac12\) in the standard softmax case.

Suppose that \(Z\) has uniformly random support of size \(\K\), and that conditioned on its support the nonzero coordinates are independent, mean zero, and with variance \(\sigma^2\). Then,
\[
\mathbb E[Z_i^2]=\frac{\K}{d}\sigma^2,
\qquad
\mathbb E[Z_iZ_j]=0 \quad (i\neq j),
\]
so
\[
\mathbb E[ZZ^\top]=\frac{\K}{d}\sigma^2 \mI.
\]
Therefore
\[
\mathbb E[\|(\mG - \mI)Z\|_2^2]
=
\mathbb E \bigl[Z^\top (\mG - \mI)^2 Z\bigr]
=
\operatorname{tr} \bigl((\mG - \mI)^2\mathbb E[ZZ^\top]\bigr)
=
\frac{\K}{d}\sigma^2\,\operatorname{tr} \bigl((\mG - \mI)^2\bigr).
\]
Since \(\mG-\mI\) is symmetric,
\[
\operatorname{tr} \bigl((\mG - \mI)^2\bigr)=\|\mG-\mI\|_F^2,
\]
and because \(\mG=\mD^\top \mD\) with unit-norm columns,
\[
\|\mG-\mI\|_F^2
=
\sum_{j = 1}^d\sum_{k\neq j}\langle \vf_k,\vf_j\rangle^2
=
d(d-1)\mu(\mD).
\]
Hence
\[
\mathbb E[\|(\mG - \mI)Z\|_2^2]
=
\K\sigma^2(d-1)\mu(\mD).
\]
Substituting this into the averaged bound of the theorem yields
\[
\frac1d\sum_{j=1}^d
\mathbb E\Bigl[\|Y^{(j,\delta)}-\widehat Y^{(j,\delta)}\|_2^2\Bigr]
\le
2L^2\|\mA\|_{\mathrm{op}}^2(d-1)\mu(\mD)\bigl(\K\sigma^2+\delta^2\bigr).
\]
In particular, in the standard softmax case,
\[
\frac1d\sum_{j=1}^d
\mathbb E\Bigl[\|Y^{(j,\delta)}-\widehat Y^{(j,\delta)}\|_2^2\Bigr]
\le
\frac12\|\mA\|_{\mathrm{op}}^2(d-1)\mu(\mD)\bigl(\K\sigma^2+\delta^2\bigr).
\]
% In this regime, both the baseline term and the intervention-specific term are governed by the same coherence quantity \(\mu(\mD)\), making the role of decoder geometry fully explicit.

\subsection{Empirical evidence}
\begin{figure}
    \centering
    \includegraphics[width=\linewidth]{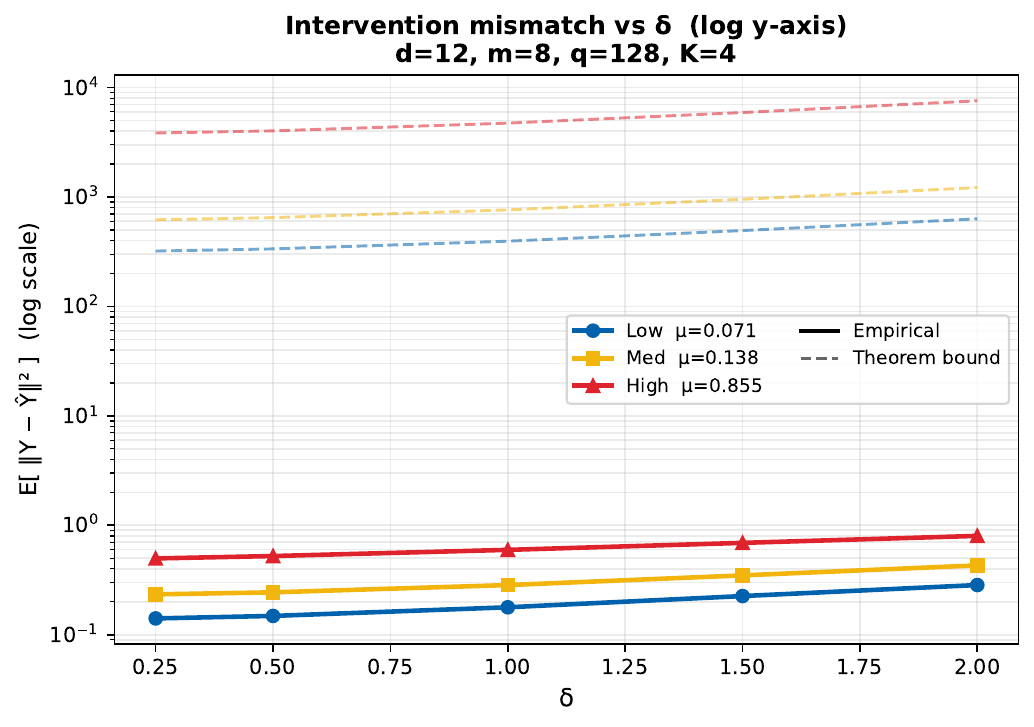}
\caption{\textbf{Intervention mismatch } We plot the discrepancy between predicted and ground truth output for different levels of self-coherence in an overcomplete dictionary. The dotted lines represent the bound derived in~\autoref{thm:upper-bound} for the three self-coherence levels. \looseness=-1}
\label{fig:mismatch}
\end{figure}
In the interference context, we use the terms "discrepancy" and "mismatch" interchangeably. \autoref{fig:mismatch} depicts the discrepancy for different levels of self-coherence. The $\medium$ self-coherence of $0.138$ has been achieved by a decoder $\mD$ with random Gaussian, L2-normalized columns. The $\high$ self-coherence of $0.855$ we obtain by creating columns which are all small perturbations of a shared direction. Constructing the overcomplete dictionary in a structured way, we reach $\low$ self-coherence of $0.071$. In addition, we display the corresponding self-coherence bounds derived in~\autoref{thm:upper-bound}. In this example, the bounds are not tight. We observe, however, that the bound for the $\high$ self-coherence setting is up to an order of magnitude above the upper bound for $\low$ self-coherence.

In spirit, this is similar to our approach of fine-tuning the SAE under an orthogonality penalty. There, too, we aim to add low self-coherence as an inductive bias to enable isolated interventions. Here, we take a full orthonormal basis in $\Rmm$ for $m = 8$ and add the remaining $d - m = 12 - 8 = 4$ columns through Givens-rotated combinations. In particular, for $j \in [d-m]$, we define the angle $\theta_j = \frac{j \pi}{d-m+1}$ and set 
\begin{align}
    \ve_j = \cos (\theta_j) \vf_{j \text{ mod } m} + \sin (\theta_j) \vf_{(j+1) \text{ mod } m}
\label{eqn:givens}
\end{align}
for $\{\vf_j\}_{j \in [m]}$ the orthonormal basis vectors. We then fill the $d-m$ columns of $\mD$ through $\{\ve_j\}_{j \in [d-m]}$ and obtain the final dictionary by renormalizing each column to unit norm.

\begin{figure}
    \centering
    \includegraphics[width=\linewidth]{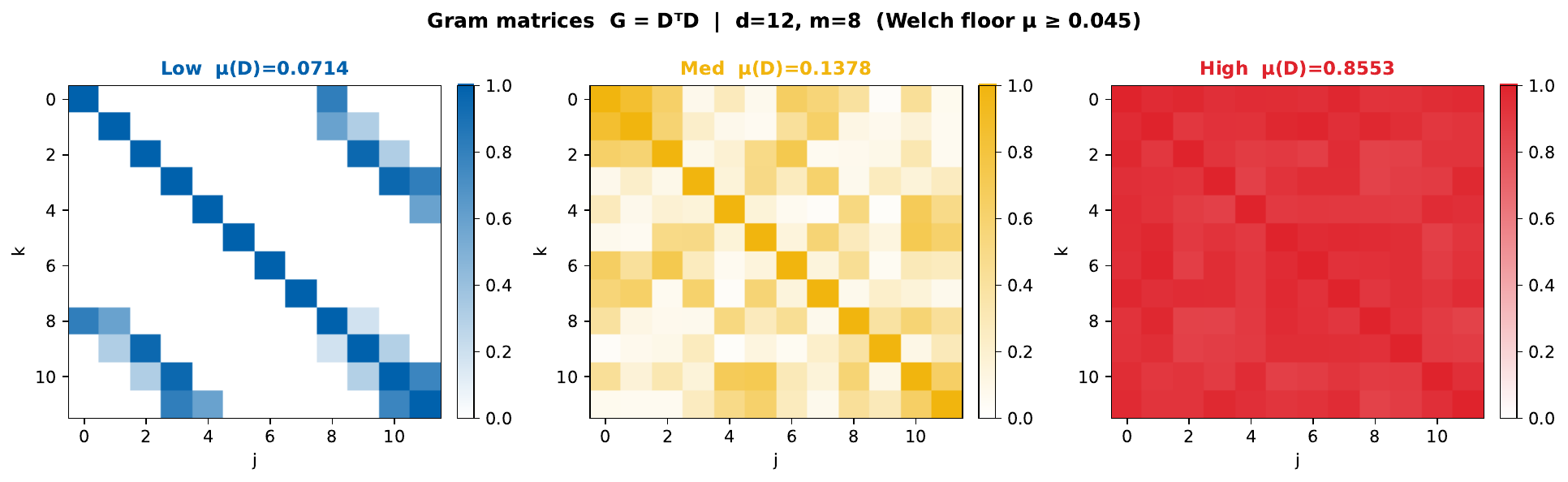}
    \caption{\textbf{Coherence of decoders for different levels of self-coherence }}
    \label{fig:coherencedecoders}
\end{figure}
Finally,~\autoref{fig:coherencedecoders} represents the gram matrix $\mG$ for each of the three levels of self-coherence as a heatmap. While the interference between features increases with self-coherence, we cab attribute the $\low$ self-coherence behavior on the off-diagonals to the generation of the dictionary $\mD$ in~\eqref{eqn:givens}.
\section{Experimental Details}
\label{app:experiments}

\subsection{Setup}
All experiments were conducted using multi-GPU training. Fine-tuning the SAE on the $\mmQA$ dataset required approximately $8$ hours using $4$ GPUs with $80$GB of memory each. Subsequent fine-tuning of the language model with the SAE integrated into the architecture required an additional $16$ hours under the same hardware configuration.

We employed parameter-efficient fine-tuning via LoRA with rank $128$ and dropout rate $0.05$. The underlying model has a hidden dimension of $2304$. Training was performed with a per-device batch size of $4$. We used an initial learning rate of $5\times10^{-5}$, combined with $200$ warm-up steps followed by a cosine decay schedule. Fine-tuning the LLM was carried out with gradient checkpointing enabled to reduce memory usage, and gradient norms were clipped to a maximum value of $1$ to ensure training stability. During generation, we greedily decode all tokens. \looseness=-1

\subsection{Interpretability}
% \begin{itemize}
%     % \item we show interpretability scores on a set similar to the orthogonality scores in 4.1
%     % \item they are (close to) significantly better using Efron basic bootstrap confidence intervals~\citep{efron1979basicbootstrap} 
%     % \item speak about superposition
%     % \item we have interpretability score as selection of 1 out of 5 with \LlamaThree
%     \item we let the description be at most $20$ words and provide $20$ spans to the model sampled from the top $100$ (as many as exist, if less) activated snippets in the $\mmQA$ testset at $95-1-4$ train-validation-test split on $395$k size dataset
%     \item the spans we let the \LlamaThree evaluate on are $\pm 10$ tokens
%     \item evaluate interpretability score on $500$ explanations with selecting the one fitting span from 5 options
%     \item prompt template
%     \item details on retrieving one of top $100$ snippets
% \end{itemize}

For $2'000$ features per orthogonality penalty, we generate a natural-language description and evaluate whether it correctly explains the feature’s activation behavior. Descriptions are constrained to a maximum length of $20$ words. For each feature, we provide the evaluator with up to $20$ text spans sampled from the top $100$ most strongly activating snippets in the $\gsmk$ test set. The dataset follows a $95$-$1$-$4$ train--validation--test split over $395'000$ examples. Each span consists of a window of $\pm 10$ tokens around the activation point.

\begin{lstlisting}[caption=\textbf{System Prompt for Interpretability Score},label={lst:system_score}]
You are evaluating the interpretability of an explanation with respect 
to a set of text snippets.

You will be given:
- A list of {total_snippets} labeled text snippets ({snippet_labels})
- A single explanation describing a specific concept, feature, or pattern

Your task is to identify which snippet the explanation applies 
to MOST STRONGLY.
Exactly {correct_snippets} of the provided snippets is correct.

Guidelines:
- Select ONLY the single snippet for which the explanation is clearly 
  and directly applicable.
- You must output exactly one label.
- Do not explain your reasoning.

Output format requirements (MANDATORY):
- Output must be a Python-style list of snippet labels.
- Example valid outputs: [1], [{total_snippets}]
- Do NOT include any additional text, punctuation, or explanation.
- Do NOT include quotes around labels.
- Do NOT include reasoning or commentary.

If the explanation best fits snippet 2, your entire output must be:
[2]
\end{lstlisting}

\label{sec:expinterp}
We evaluate interpretability over $300$ feature explanations using a multiple-choice setup, where the evaluator selects the span best matching the provided description from five candidates. All evaluations are performed using \LlamaThree. We include the full system prompt in \autoref{lst:system_score}. A representative example of feature explanation is given in~\autoref{lst:aquafive}.
\begin{wrapfigure}{l}{0.5\linewidth}
    \centering
    \includegraphics[width=\linewidth]{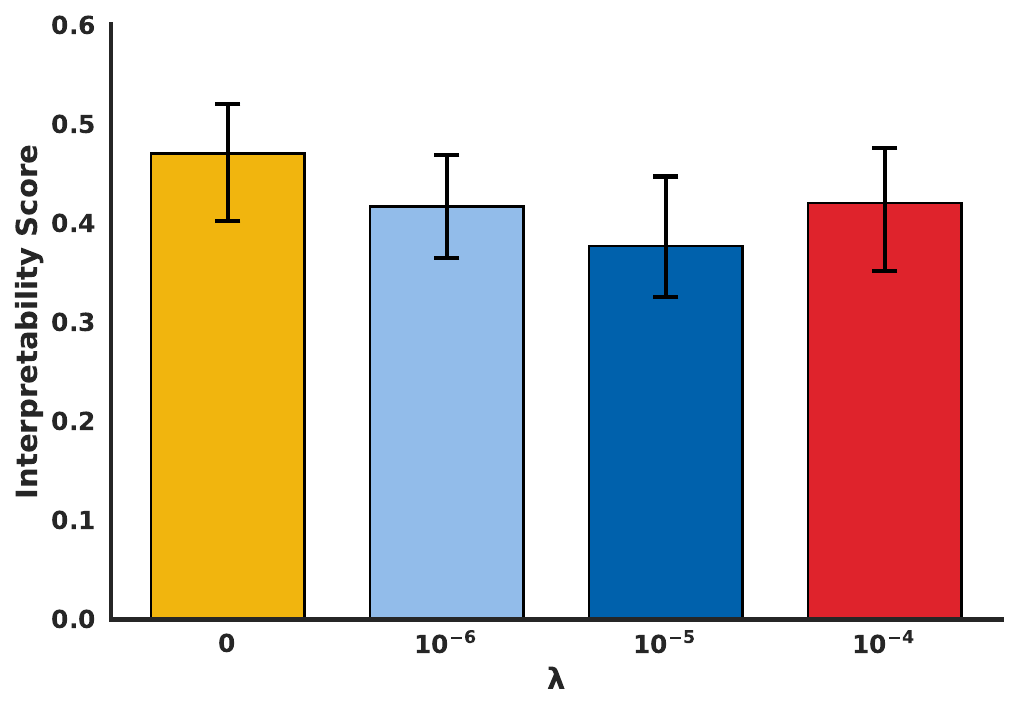}
    \caption{\textbf{Interpretability score } We plot the interpretability score of correctly identifying one out of five examples, which relates closest to the provided explanation. Error bars are basic bootstrap confidence intervals~\citep{efron1979basicbootstrap} with $100$ resamples.}
    \label{fig:interpretability}
\end{wrapfigure}

\subsection{Interpretability results}
SAEs are primarily used for mechanistic interpretability. By uncovering human-understandable concepts, we aim to track the Transformer's internal computations. We therefore check the interpretability of our SAEs under different orthogonality penalties. For this, we create LM-generated explanations for $2'000$ activated features in each of our fine-tuned SAEs. We do so providing $20$ text snippets which activate the respective feature and asking \LlamaThree to describe the common theme across snippets. For details and prompt templates, see below.

We then ask \LlamaThree to match one of five text snippets to the corresponding feature explanation. \autoref{fig:interpretability} demonstrates our results. Without allowing the model to reason, we observe interpretability performance on this task between $37\%$ and $47\%$. Given the basic bootstrap error bars~\citep{efron1979basicbootstrap}, we observe non-significant difference between orthogonality parameters. We also notice significantly better interpretability than random guessing at $20\%$. \looseness=-1

Finally, we note that any feature we report has the same feature index across all tested values of $\lambda$. Hence, LoRA adapts the LM around the SAE without repurposing feature indices. This indicates that the orthogonality regularization reshapes the geometric relationships between features while leaving the encoded concepts unchanged. Together with the stable performance in~\autoref{fig:mathevaluation} and interpretability scores in~\autoref{fig:interpretability}, this provides evidence that our pipeline does not induce unintended semantic drift.

\begin{lstlisting}[caption=\textbf{System Prompt for Feature Explanations},label={lst:system}]
You are labeling latent features of a language model trained on 
grade-school math word problems.

You will see {num_spans} numbered spans from these problems. In each 
span, exactly one token is wrapped in double square brackets, like 
[[this]]. The feature you are labeling fires on precisely the bracketed 
token; the surrounding text is context only.

Procedure:
1. Resolve each bracketed token to the full word or entity it belongs 
   to. A bracketed fragment of a longer word stands for the whole word: 
   [[umbr]] inside "umbrella" means umbrella.
2. Find the single most specific concept shared by these resolved 
   words across (nearly) all {num_spans} spans.
3. Name that concept as specifically as the evidence allows: 
   a concrete entity (a particular first name, animal, food 
   or drink, place, object); else a specific mathematical element
   (a particular unit, operation word, number format, quantity type); 
   else a structural role.

Constraints:
- Your answer must name something you can point to in the bracketed 
  tokens themselves. Never answer with a word that does not appear 
  in, or directly behind, the bracketed tokens.
- The illustration word in these instructions ("umbrella") is not 
  evidence; no word from these instructions may appear in your answer 
  unless it genuinely occurs in the spans.
- State the most specific level the spans support: the particular 
  name, animal, or unit itself, not its category 
  ("first names", "animals", "units").
- Do NOT describe features as "mathematical", "math-related", 
  "word problems", "numbers", or "reasoning".
- Do not mention the brackets, tokens, models, or 
  activations in your answer.
- Output exactly one sentence (max {max_words} words) in the 
  exact form: "The shared latent concept is ...".
- Your sentence must contain the specific shared word or name 
  itself, copied verbatim from the spans and wrapped in 
  quotation marks - category plus instance, e.g.: The shared 
  latent concept is the first name "X". / the unit "X". / 
  the animal "X".
- If the bracketed tokens resolve to different words that 
  share a category, name the category and quote two or three 
  of the actual words, e.g.: first names such as "X" and "Y".
- A category alone ("a first name", "an animal", "a unit") is 
  never an acceptable answer. If you cannot quote the specific 
  word(s), reply exactly: "No coherent concept."
\end{lstlisting}

% \begin{itemize}
%     \item we retrieve explanations for $1000$ features in each SAE
%     \item explain how we arrive at $1'000$ from $2'000$
%     \item active features are less in the more orthogonal case, hypothesis that concepts are put together 
%     \item we take cosine similarity between active features
%     \item we show that explanations for more orthogonal models exhibit lower cosine similarity
%     \item choose same models as above
%     \item we restrict ourselves to explanations that start in one of 5 ways, but these encompass $x\%$ in the explanations
%     \item details on this are in appendix
%     \item all prompts are in appendix
%     \item show one example of a feature explanation, e.g., aquarium
%     \item discuss the complexity of the dataset and the number of features present
% \end{itemize}

\subsection{Intervenability}
To evaluate the extent to which learned features support localized and controllable interventions, we conduct targeted manipulation experiments. We identify $12$ SAE features corresponding to male first names. We observe that the model assigns the same feature indices to different variations of a name. We thus list in~\autoref{lst:names} all variations of names which we accept as correct. As our intervenability study hinges on string extraction, the spelling of the concept is relevant. Across all models considered, these concepts are consistently represented by the same feature indices. Before running the experiment, we check that all names are recoverable in at least some contexts.

For each name, we intervene by swapping the corresponding feature activation at every token position where the name appears. We sweep over the following values for the associated coefficient $z_j, j \in [d]$,
\begin{align}
    \{10, 20, 50, 100, 150, 200, 300, 500\}.
\label{eqn:sweep}
\end{align}

By doing so, we find that the accuracy ranges from $0.601$ to $0.658$ for any combination of insertion value and orthogonality penalty. We then decide the hyperparameter based on the models' ability to remove the intervened on concept. For sufficiently high insertion value, the model near-perfectly drops the intervened on concept from the generation. Hence, we choose among the set of hyperparameters that includes at most $1\%$ of concepts we intervene on. This leaves us with hyperparameter values greater than or equal to $100$. For our experiments, we choose hyperparameter $200$. We show similar behavior for $z_j \in \{100, 150, 300\}$ in~\autoref{fig:intervenability_names_sweep}.

\begin{figure*}
     \centering
     \begin{subfigure}[b]{0.49\textwidth}
         \centering
         \includegraphics[width=\textwidth]{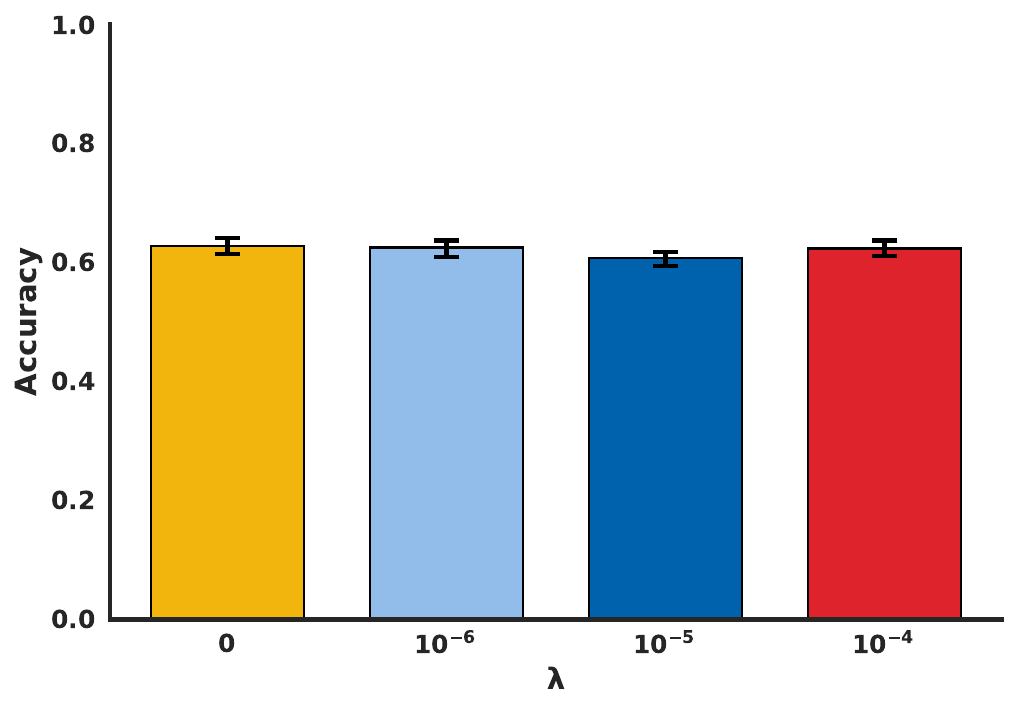}
         \caption{\textbf{Evaluation on mathematical reasoning } We plot mathematical reasoning performance for insertion hyperparameter $100$.}
         \label{fig:intervenability_eval_100}
     \end{subfigure}
     \hfill
     \begin{subfigure}[b]{0.49\textwidth}
         \centering
         \includegraphics[width=\textwidth]{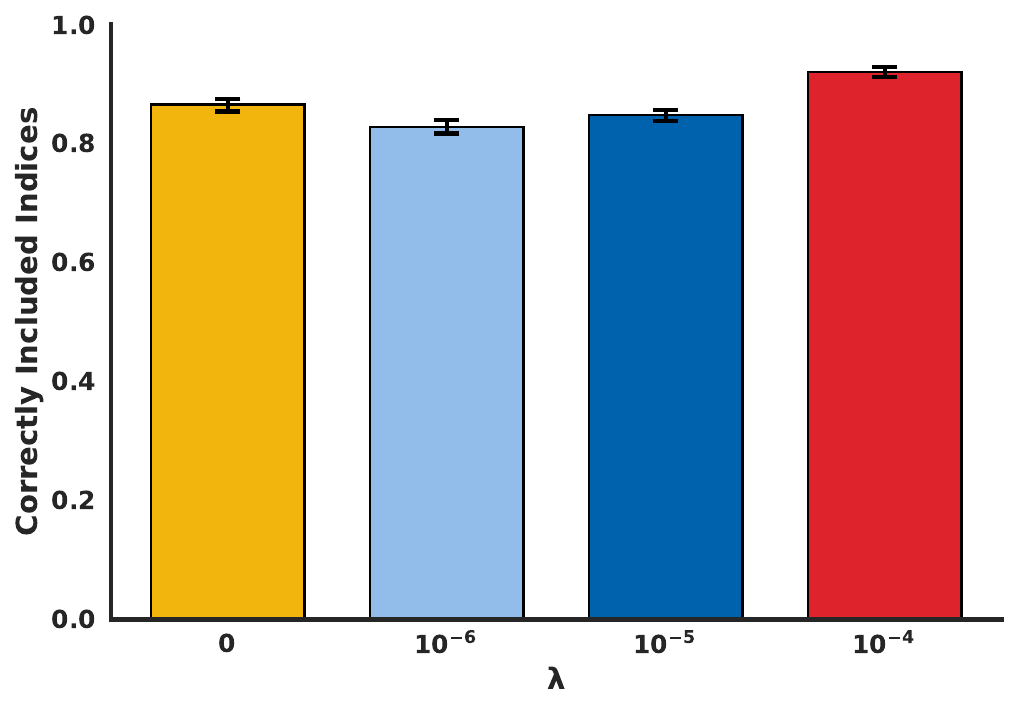}
         \caption{\textbf{Evaluation on inclusion of correct name } We plot the fraction of correctly included names for insertion hyperparameter $100$.}
         \label{fig:intervenability_include_100}
     \end{subfigure}
     \par\vspace{2em}
     \begin{subfigure}[b]{0.49\textwidth}
         \centering
         \includegraphics[width=\textwidth]{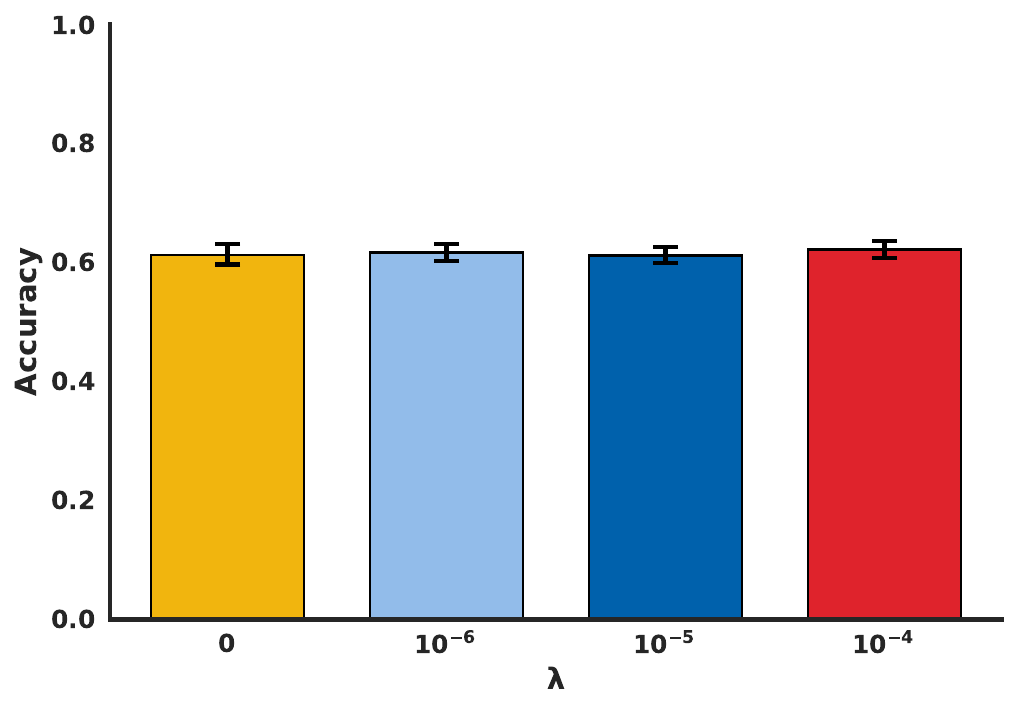}
         \caption{\textbf{Evaluation on mathematical reasoning } We plot mathematical reasoning performance for insertion hyperparameter $150$.}
         \label{fig:intervenability_eval_150}
     \end{subfigure}
     \hfill
     \begin{subfigure}[b]{0.49\textwidth}
         \centering
         \includegraphics[width=\textwidth]{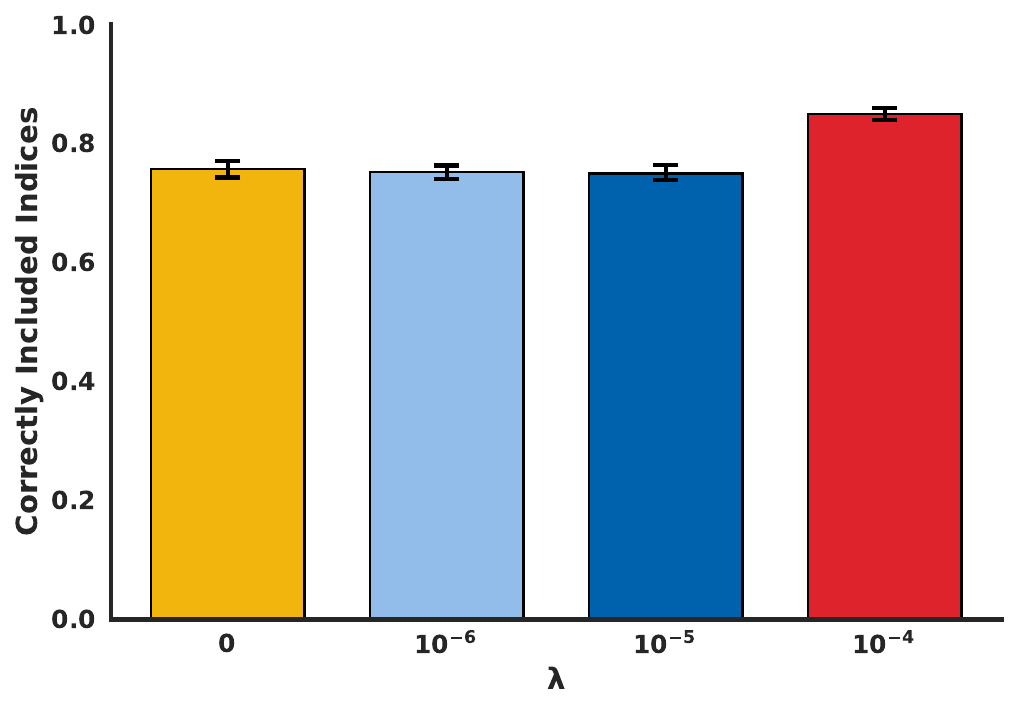}
         \caption{\textbf{Evaluation on inclusion of correct name } We plot the fraction of correctly included names for insertion hyperparameter $150$.}
         \label{fig:intervenability_include_150}
     \end{subfigure}
     \par\vspace{2em}
     \begin{subfigure}[b]{0.49\textwidth}
         \centering
         \includegraphics[width=\textwidth]{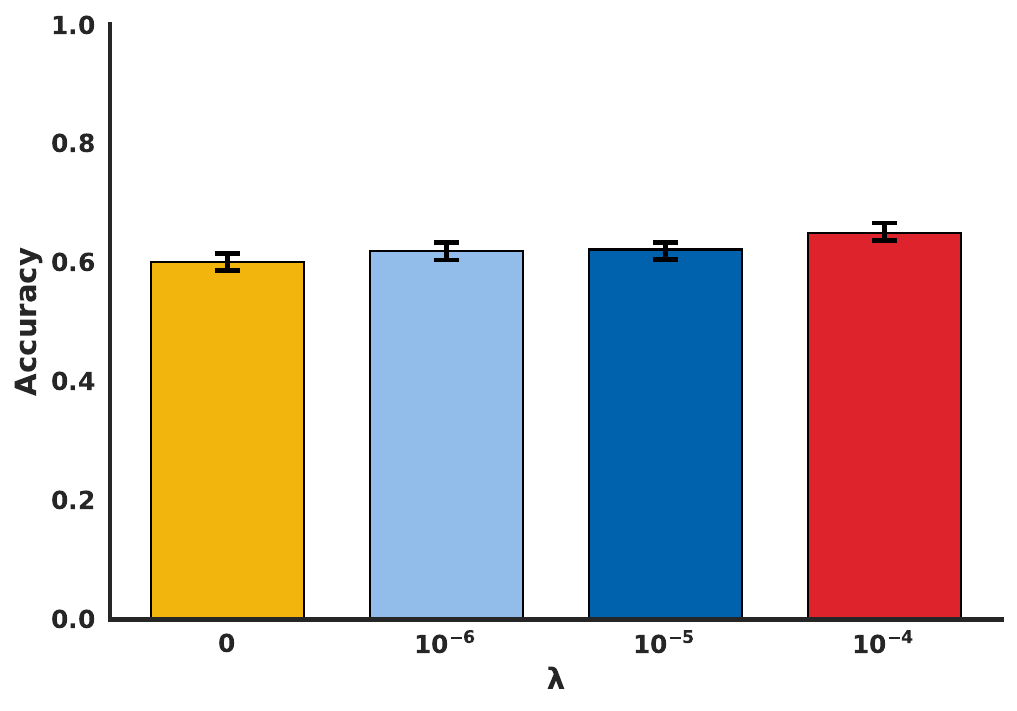}
         \caption{\textbf{Evaluation on mathematical reasoning } We plot mathematical reasoning performance for insertion hyperparameter $300$.}
         \label{fig:intervenability_eval_300}
     \end{subfigure}
     \hfill
     \begin{subfigure}[b]{0.49\textwidth}
         \centering
         \includegraphics[width=\textwidth]{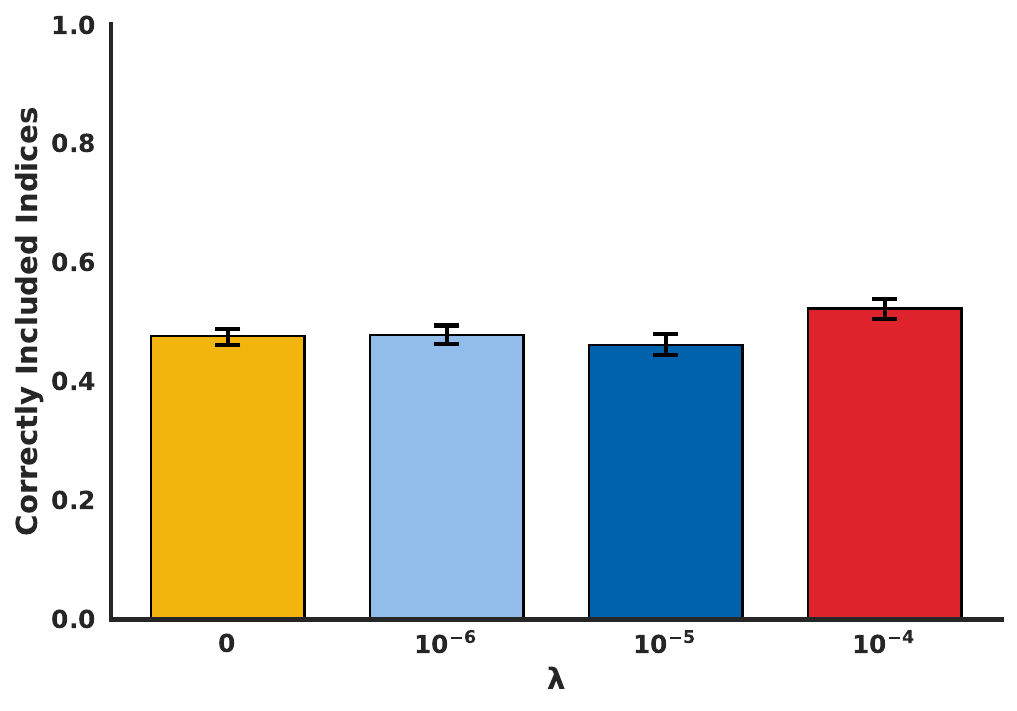}
         \caption{\textbf{Evaluation on inclusion of correct name } We plot the fraction of correctly included names for insertion hyperparameter $300$.}
         \label{fig:intervenability_include_300}
     \end{subfigure}
    \caption{\textbf{Intervenability across hyperparameter sweep } We intervene on $12$ features in the SAE corresponding to the concept of first names. Then, we evaluate the ability to maintain reasoning performance and the correct generation of the included first name.}
    \label{fig:intervenability_names_sweep}
\end{figure*}

Choosing $z_j = 200$, we evaluate both \emph{drop} and \emph{include} interventions. While all models successfully suppress the original concept under drop interventions, including a new concept proves more challenging for models with weaker orthogonality. We hypothesize that this asymmetry arises from differences in the alignment of feature directions within the embedding space. Beyond names, we also observe successful semantic interventions on other concepts; for example, activating an $\aqua$ feature causes the model to transform the character $\Mike$ into $\Aquaman$ while leaving the surrounding context unchanged. \looseness=-1

\begin{lstlisting}[caption={\textbf{First names used for interventions}},label={lst:names}]
{
"Jason":    ["Jason", "Jase"],
"Mike":     ["Michael", "Mike", "Mikey"],
"Jacob":    ["Jacob", "Jake", "Jakob"],
"Jerry":    ["Jerry", "Jeremy", "Jermey", "Jerome"],
"James":    ["James", "Jim", "Jimmy", "Jamie"],
"Robert":   ["Robert", "Rob", "Robbie"],
"Jordan":   ["Jordan", "Jordy"],
"Jackson":  ["Jackson", "Jack", "Jax"],
"Paul":     ["Paul", "Pauly"],
"David":    ["David", "Dave", "Davey"],
"Andrew":   ["Andrew", "Andy"],
"Gary":     ["Gary", "Garre", "Garret", "Garrett", "Garry"]
}
\end{lstlisting}

\subsection{Intervenability on animals and locations}
\begin{figure*}
     \centering
     \begin{subfigure}[b]{0.49\textwidth}
         \centering
         \includegraphics[width=\textwidth]{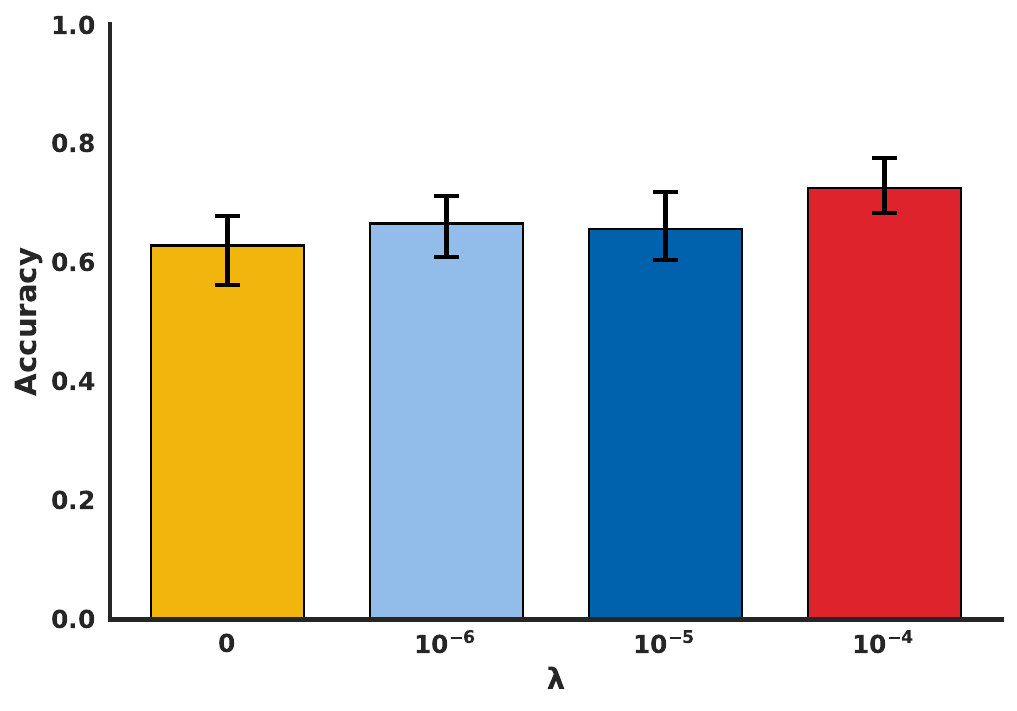}
         \caption{\textbf{Evaluation on mathematical reasoning } We plot mathematical reasoning performance on $\animal$-related queries after intervening on the SAE. Error bars are basic bootstrap confidence intervals~\citep{efron1979basicbootstrap} with $100$ samples.}
         \label{fig:intervenability_eval_animal}
     \end{subfigure}
     \hfill
     \begin{subfigure}[b]{0.49\textwidth}
         \centering
         \includegraphics[width=\textwidth]{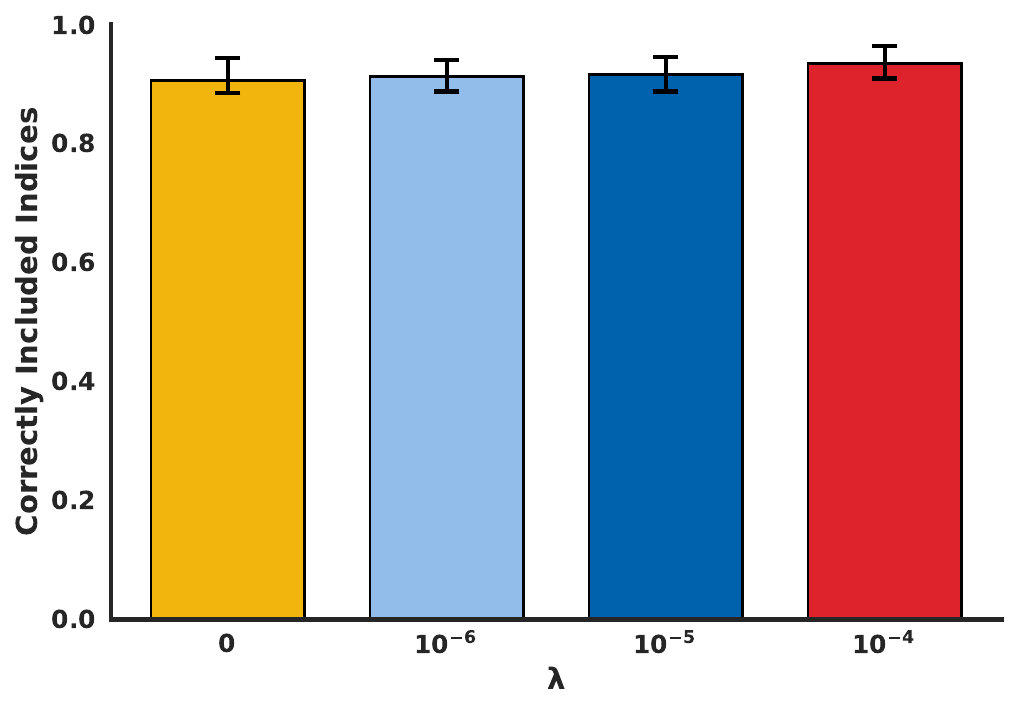}
         \caption{\textbf{Evaluation on inclusion of correct $\animal$ } We plot the fraction of correctly included $\animal$ names after intervening on the corresponding feature in the SAE. Error bars are basic bootstrap confidence intervals of $100$ random draws.}
         \label{fig:intervenability_include_animal}
     \end{subfigure}
    \caption{\textbf{Intervenability } We intervene on $5$ features in the SAE corresponding to the concept of $\animal$ names. Then, we evaluate the ability to maintain reasoning performance and the correct generation of the included $\animal$ name. We choose an insertion value of $200$. \looseness=-1}
    \label{fig:intervenability_animal}
\end{figure*}

\begin{lstlisting}[caption={\textbf{Additional concepts used for interventions}},label={lst:animlocdrnk}]
{
"animals": ["dog", "rabbit", "bear", "goat", "horse"],
"locations": ["park", "room", "school", "library", "fast-food restaurant"],
}
\end{lstlisting}
Similar to our analysis on first names, we run experiments on two additional groups. We notice now that for sentences where correct retrieval is simple for all models, the model with the strictest orthogonality penalty of $10^{-4}$ obtains higher accuracy than the non-penalized model. This can be seen for the $\animal$ dataset at insertion value $200$ in~\autoref{fig:intervenability_animal}, for instance, at an accuracy of $72.5\%$ compared to $62.8\%$. The five $\animal$ concepts are shown in~\autoref{lst:animlocdrnk}. While slightly higher at $93.4\%$ relative to $90.6\%$, correct feature inclusion between strictest penalty and no-penalty setup are comparable. Under all setups, all features corresponding to dropped concepts are adequately excluded.

The results for the $\location$ concept are found in~\autoref{tab:location}. On the one hand, we observe above-average evaluation scores of up to $95.7\%$ for $\lambda = 10^{-4}$ at modest insertion value of $10$. We suspect this to be an artifact of our dataset generation. As several sentences fall into the category of part-whole or other simple arithmetic problems, the evaluation scores tend to be higher than for the full dataset in~\autoref{fig:intervenability}. We expect these problems to be easier to solve for the LM. \looseness=-1

On the other hand, the inclusion values for the $\location$ concept are relatively small. Here, we visually inspect the ground truth reasoning traces in $\gsmk$ of the original sentences. We notice that the answer often does not restate concept in the corpus. In other words, the model is not necessarily fine-tuned on restating the concept explicitly to arrive at the final answer. As the location regularly serves as the place description in the sentence, it is only tangentially relevant to the computation of the right solution. Therefore, it is natural to omit this during generation. However, we find these results relevant because they enhance our understanding beyond the subject-level first names. We refer the reader to~\autoref{lst:intervdatasetlocation}. \looseness=-1

\subsection{Behavioral and internal measures of interference}
We extend our above analysis to the $\animal$ and $\location$ dataset. The results are reported in Tables \ref{tab:rouge-jsd-animal} and \ref{tab:rouge-jsd-locations}.

\begin{table}[h!]
\centering
\begin{tabular}{lcc cc}
\toprule
& \multicolumn{2}{c}{$\ROUGE$} & \multicolumn{2}{c}{$\JSD$} \\
\cmidrule(lr){2-3}
\cmidrule(lr){4-5}
$\lambda$ & Mean & Confidence interval & Mean & Confidence interval \\
\midrule
$0$ & 0.761 & [0.741, 0.774] & 0.259 & [0.248, 0.272] \\
$10^{-6}$ & 0.793 & [0.780, 0.804] & 0.243 & [0.233, 0.252] \\
$10^{-5}$ & 0.797 & [0.785, 0.809] & 0.229 & [0.220, 0.241] \\
$10^{-4}$ & 0.819 & [0.803, 0.835] & 0.246 & [0.234, 0.259] \\

\bottomrule
\end{tabular}
\caption{\textbf{Results for interference measures } We evaluate $\ROUGE$ and $\JSD$ on the tokens and features of the $\animal$ dataset. We report mean and the basic bootstrap confidence interval over all observations.}
\label{tab:rouge-jsd-animal}
\end{table}

\begin{table}[h!]
\centering
\begin{tabular}{lcc cc}
\toprule
& \multicolumn{2}{c}{$\ROUGE$} & \multicolumn{2}{c}{$\JSD$} \\
\cmidrule(lr){2-3}
\cmidrule(lr){4-5}
$\lambda$ & Mean & Confidence interval & Mean & Confidence interval \\
\midrule
$0$ & 0.898 & [0.879, 0.914] & 0.236 & [0.202, 0.269] \\
$10^{-6}$ & 0.894 & [0.867, 0.923] & 0.246 & [0.214, 0.276] \\
$10^{-5}$ & 0.871 & [0.843, 0.893] & 0.251 & [0.220, 0.278] \\
$10^{-4}$ & 0.894 & [0.879, 0.905] & 0.182 & [0.167, 0.198] \\

\bottomrule
\end{tabular}
\caption{\textbf{Results for interference measures } We evaluate $\ROUGE$ and $\JSD$ on the tokens and features of the $\location$ dataset. We report mean and the basic bootstrap confidence interval over all observations.}
\label{tab:rouge-jsd-locations}
\end{table}

\subsection{Dead features}
\autoref{fig:dead_fraction} demonstrates that the configuration with orthogonality penalty $\lambda = 10^{-4}$ only uses approximately two thirds of the features of the non-regularized configuration. In Section~\ref{sec:discussion}, we hypothesize that this constitutes an inherent characteristic of the strict orthogonality penalty. \looseness=-1

\begin{figure}[h]
    \centering
    \includegraphics[width=0.5\linewidth]{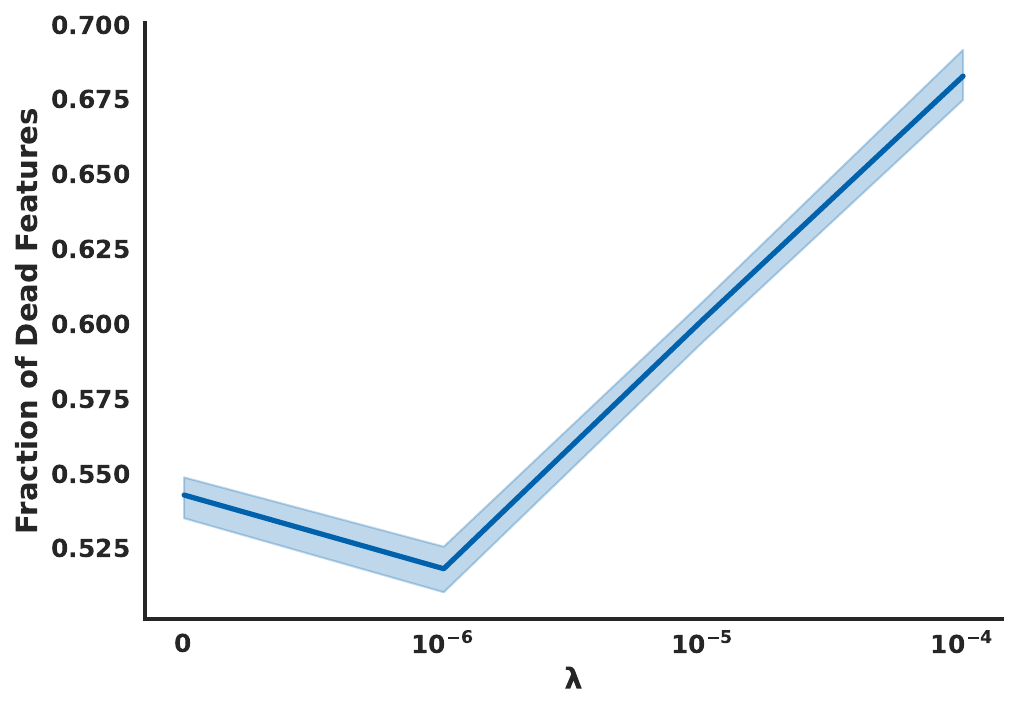}
    \caption{\textbf{Dead features } We plot the fraction of dead features with increasing orthogonality penalty. Error bars are basic bootstrap confidence intervals~\citep{efron1979basicbootstrap}.}
    \label{fig:dead_fraction}
\end{figure}

\subsection{Experiments on \Llama}
Building on the above experiments, we run the pipeline on \Llama to confirm our results. Our starting point is the FaithfulSAE~\citep{cho2025faithfulsae} optimized for $\K=48$. For easier comparability, we still run our two-stage fine-tuning pipeline on $\K=20$. Similar to~\citet{cho2025faithfulsae}, we insert the SAE after layer $12$. Architecturally, \Llama consists of $16$ layers instead of the $26$ \GemmaTwoB layers. Further, the SAE dimension is lower at $14'336$ compared to $65'536$. On top of intervening on one of the layers in the early half of the Gemma-2-2B transformer, the \Llama experiments strengthen our claim by inserting the SAE three-quarters of the way through the network. \looseness=-1

\begin{figure*}[h]
     \centering
     \begin{subfigure}[b]{0.32\textwidth}
         \centering
         \includegraphics[width=\textwidth]{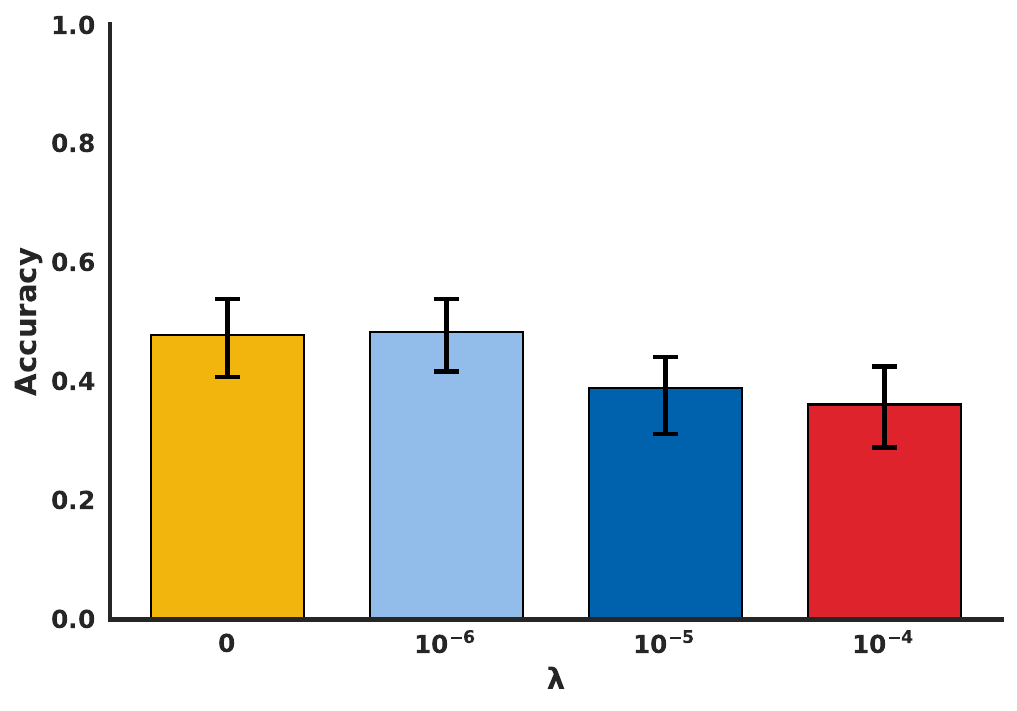}
         \caption{\textbf{Evaluation on mathematical reasoning } We plot mathematical reasoning performance after intervening on the SAE.}
         \label{fig:intervenability_eval_llama}
     \end{subfigure}
     \hfill
     \begin{subfigure}[b]{0.32\textwidth}
         \centering
         \includegraphics[width=\textwidth]{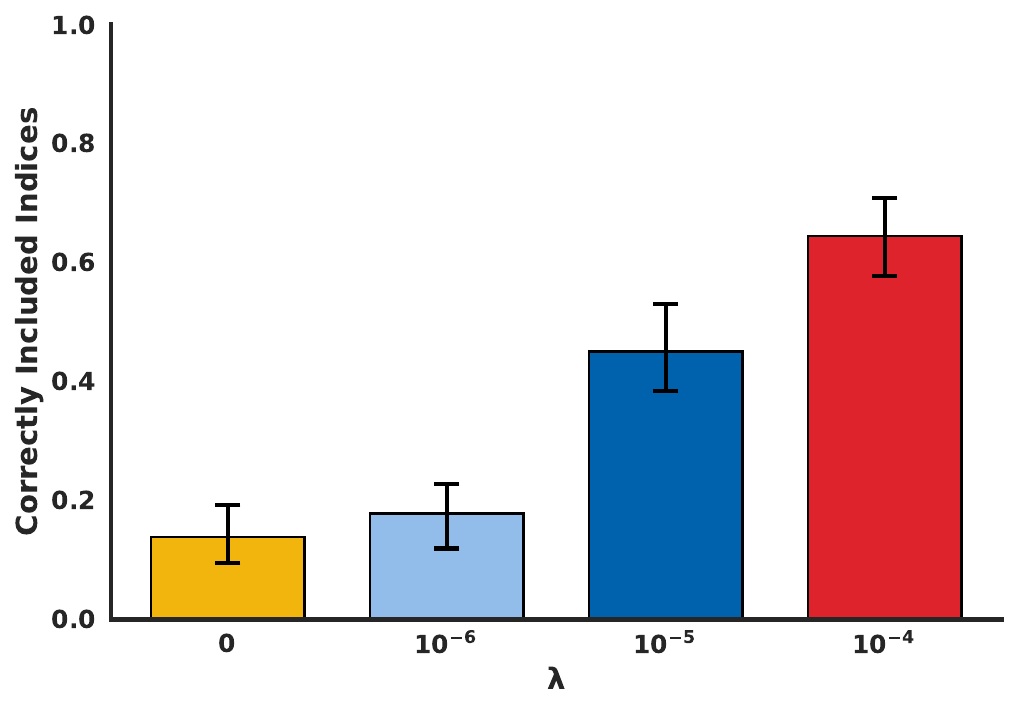}
         \caption{\textbf{Evaluation on including correct name } We plot the fraction of correctly included first names after intervention in the SAE. \looseness=-1}
         \label{fig:intervenability_include_llama}
     \end{subfigure}
     \hfill
     \begin{subfigure}[b]{0.32\textwidth}
         \centering
         \includegraphics[width=\textwidth]{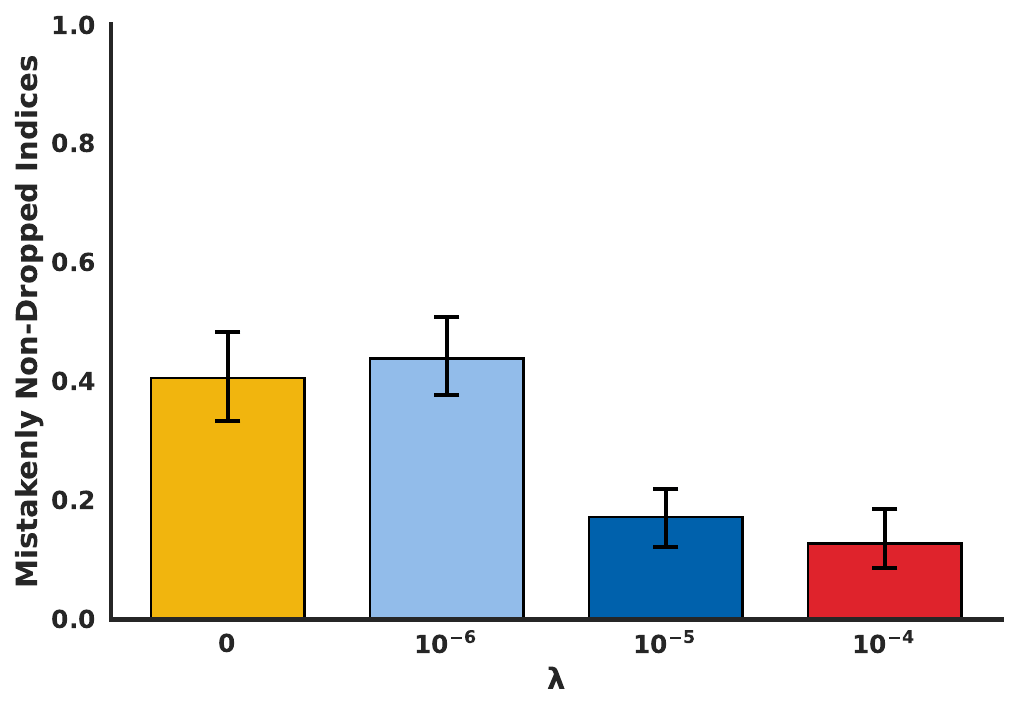}
         \caption{\textbf{Evaluation on dropping the first name } We plot the fraction of mistakenly included first names after dropping feature. \looseness=-1}
         \label{fig:intervenability_drop_llama}
     \end{subfigure}
    \caption{\textbf{Intervenability } We intervene on $180$ examples containing male first names. Then, we evaluate the ability to maintain reasoning performance and the correct generation of the included first name. We choose an insertion value of $20$. Error bars are basic bootstrap confidence intervals~\citep{efron1979basicbootstrap} with $100$ samples. \looseness=-1}
    \label{fig:intervenability_llama}
\end{figure*}

In contrast to the \GemmaTwoB experiments, we observe non-zero inclusion of the intervened on features for all orthogonality penalties and insertion values. We therefore do not compare the models conditional on properly excluding the dropped feature but choose the insertion value which yields highest accuracy. We thus also report that statistic in~\autoref{fig:intervenability_drop_llama}. Sweeping over the values for the coefficient associated with the intervention feature indices in~\eqref{eqn:sweep}, we note highest accuracy for an insertion value of $20$ for all configurations except $\lambda = 10^{-5}$, for which this is $10$. We thus choose $20$ but observe similar results for $10$. We again intervene on gendered first names and evaluate on $180$ snippets. In line with the relatively smaller model, we observe lower accuracy across all configurations. \autoref{fig:intervenability_eval_llama} indicates better performance for the no-penalty and $\lambda = 10^{-6}$ configurations compared to stricter setups. Due to the width of the basic bootstrap confidence intervals, however, we cannot claim significantly worse performance of the $\lambda = 10^{-4}$ configuration. Primarily interested in intervenability, we note in~\autoref{fig:intervenability_include_llama} that first name inclusion monotonically increases with stricter orthogonality penalty. According to the basic bootstrap intervals reported, this is also significant. Moreover,~\autoref{fig:intervenability_drop_llama} plots the fraction of included first names after removing the corresponding SAE feature from the residual stream. Note that lower inclusion of incorrect indices indicates crisper interventions. While non-zero, the two strictest configurations drop significantly more first names than the settings with $\lambda \in \{0, 10^{-6}\}$.

We understand this as additional support for our claim that orthogonality regularization improves intervenability inside the SAE. Tested on two different models at different SAE sizes and positions, we obtain qualitatively similar results. On the one hand, mathematical reasoning accuracy neither improves nor falls significantly. Interventions, on the other hand, are significantly crisper for stricter orthogonality. We show this both in terms of dropping the intervened on SAE feature as well as including the feature we manually add to the residual stream. \looseness=-1

\newpage
\subsection{Intervenability datasets}
\begin{lstlisting}[caption={\textbf{Intervenability dataset for first names}},label={lst:intervdataset}]
{
"0": {"sentence": [f"{first_name} buys twice as many red ties as blue ties. The red ties cost 50% more than blue ties. He spent $200 on blue ties that cost $40 each. How much did he spend on ties?"], "ground_truth": 800},
"1": {"sentence": [f"{first_name} sold clips to 48 of his friends in April, and then he sold half as many clips in May. How many clips did {first_name} sell altogether in April and May?"], "ground_truth": 72,},
"2": {"sentence": [f"{first_name} writes a 3-page letter to 2 different friends twice a week. How many pages does he write a year?"], "ground_truth": 624,},
"3": {"sentence": [f"{first_name} is wondering how much pizza he can eat in one day. He buys 2 large pizzas and 2 small pizzas. A large pizza has 16 slices and a small pizza has 8 slices. If he eats it all, how many pieces does he eat that day?"], "ground_truth": 48,},
"4": {"sentence": [f"{first_name} created a care package to send to his brother, who was away at boarding school. {first_name} placed a box on a scale, and then he poured into the box enough jelly beans to bring the weight to 2 pounds. Then, he added enough brownies to cause the weight to triple. Next, he added another 2 pounds of jelly beans. And finally, he added enough gummy worms to double the weight once again. What was the final weight of the box of goodies, in pounds?"], "ground_truth": 16,},
"5": {"sentence": [f"{first_name} can read 8 pages of a book in 20 minutes. How many hours will it take him to read 120 pages?"], "ground_truth": 5,},
"6": {"sentence": [f"{first_name} is fundraising for his charity by selling brownies for $3 a slice and cheesecakes for $4 a slice. If {first_name} sells 43 brownies and 23 slices of cheesecake, how much money does {first_name} raise?"], "ground_truth": 221,},
"7": {"sentence": [f"{first_name} is putting together 4 tables. Each table has 4 legs and each leg needs 2 screws. He has 40 screws. How many screws will he have left over?"], "ground_truth": 8,},
"8": {"sentence": [f"{first_name} buys 10 packs of magic cards. Each pack has 20 cards and 1/4 of those cards are uncommon. How many uncommon cards did he get?"], "ground_truth": 50,},
"9": {"sentence": [f"{first_name} creates a media empire. He creates a movie for $2000. Each DVD cost $6 to make. He sells it for 2.5 times that much. He sells 500 movies a day for 5 days a week. How much profit does he make in 20 weeks?"], "ground_truth": 448000,},
"10": {"sentence": [f"{first_name} went to a shop to buy some groceries. He bought some bread for $2, butter for $3, and juice for two times the price of the bread. He had $15 for his shopping. How much money did {first_name} have left?"], "ground_truth": 6,},
"11": {"sentence": [f"{first_name} has 2 dogs, 3 cats and twice as many fish as cats and dogs combined. How many pets does {first_name} have in total?"], "ground_truth": 15,},
"12": {"sentence": [f"{first_name} is buying a new pair of shoes that costs $95. He has been saving up his money each month for the past three months. He gets a $5 allowance a month. He also mows lawns and shovels driveways. He charges $15 to mow a lawn and $7 to shovel. After buying the shoes, he has $15 in change. If he mows 4 lawns, how many driveways did he shovel?"], "ground_truth": 5,},
"13": {"sentence": [f"{first_name} is three times older than Ben. Ben is two times older than Chris. If Chris is 4, how old is Caroline?"], "ground_truth": 24,},
"14": {"sentence": [f"{first_name} eats 1 apple a day for two weeks. Over the next three weeks, he eats the same number of apples as the total of the first two weeks. Over the next two weeks, he eats 3 apples a day. Over these 7 weeks, how many apples does he average a week?"], "ground_truth": 10,},
"15": {"sentence": [f"{first_name} bought 2 soft drinks for$ 4 each and 5 candy bars. He spent a total of 28 dollars. How much did each candy bar cost?"], "ground_truth": 4,},
"16": {"sentence": [f"{first_name} has a stack of books that is 12 inches thick. He knows from experience that 80 pages is one inch thick. If he has 6 books, how many pages is each one on average?"], "ground_truth": 160,},
"17": {"sentence": [f"{first_name} is throwing a huge Christmas party. He invites 30 people. Everyone attends the party, and half of the guests bring a plus one (one other person). He plans to serve a 3-course meal for the guests. If he uses a new plate for every course, how many plates does he need in total for his guests?"], "ground_truth": 135,},
"18": {"sentence": [f"{first_name} volunteers at a shelter twice a month for 3 hours at a time. How many hours does he volunteer per year?"], "ground_truth": 72,},
"19": {"sentence": [f"{first_name} puts $25 in his piggy bank every month for 2 years to save up for a vacation. He had to spend $400 from his piggy bank savings last week to repair his car. How many dollars are left in his piggy bank?"], "ground_truth": 200,},
"20": {"sentence": [f"{first_name} has 16 toy cars, and the number of cars he has increases by 50% every year. How many toy cars will {first_name} have in three years?"], "ground_truth": 54,},
"21": {"sentence": [f"{first_name} just turned 12 and started playing the piano. His friend Sheila told him about the 10,000-hour rule which says, after 10,000 hours of practice, you become an expert or master in your field. If {first_name} wants to become a piano expert before he is 20, how many hours a day will he need to practice if he practices every day, Monday - Friday, and takes two weeks off for vacation each year?"], "ground_truth": 5,},
"22": {"sentence": [f"In 6 months Bella and {first_name} will be celebrating their 4th anniversary. How many months ago did they celebrate their 2nd anniversary?"], "ground_truth": 18,},
"23": {"sentence": [f"{first_name} was preparing for a party to be held in four days. So, he made 24 gallons of root beer on the first day and put them in the refrigerator cooler. But later that evening, his children discovered the delicious nectar and robbed the cooler, drinking 4 of those gallons of root beer. On the second day, his wife Barbie also discovered the root beer and accidentally spilled 7 gallons. On the third day, {first_name}'s friend Ronnie visited {first_name}'s house and helped himself to the root beer, further reducing the amount remaining by 5 gallons. On the fourth day, 3 people showed up for the party. If {first_name} and the others shared the remaining root beer equally, how much was available for each to drink during the party?"], "ground_truth": 350,},
"24": {"sentence": [f"{first_name} owns an ice cream shop and every sixth customer gets a free ice cream cone. Cones cost $2 each. If he sold $100 worth of cones, how many free ones did he give away?"], "ground_truth": 10,},
"25": {"sentence": [f"{first_name} likes to collect model trains. He asks for one for every birthday of his, and asks for two each Christmas. {first_name} always gets the gifts he asks for, and asks for these same gifts every year for 5 years. At the end of the 5 years, his parents give him double the number of trains he already has. How many trains does {first_name} have now?"], "ground_truth": 45,},
"26": {"sentence": [f"{first_name} has $5000. He spends $2800 on a new motorcycle, and then spends half of what's left on a concert ticket. {first_name} then loses a fourth of what he has left. How much money does he have left?"], "ground_truth": 825,},
"27": {"sentence": [f"{first_name} has some coins. He has 2 more quarters than nickels and 4 more dimes than quarters. If he has 6 nickels, how much money does he have?"], "ground_truth": 350,},
"28": {"sentence": [f"{first_name} is at a train station and is waiting for his train. He isn't sure how long he needs to wait, but he knows that the fourth train scheduled to arrive at the station is the one he needs to get on. The first train is scheduled to arrive in 10 minutes, and this train will stay in the station for 20 minutes. The second train is to arrive half an hour after the first train leaves the station, and this second train will stay in the station for a quarter of the amount of time that the first train stayed in the station. The third train is to arrive an hour after the second train leaves the station, and this third train is to leave the station immediately after it arrives. The fourth train will arrive 20 minutes after the third train leaves, and this is the train {first_name} will board. In total, how long, in minutes, will {first_name} wait for his train?"], "ground_truth": 27,},
"29": {"sentence": [f"{first_name} decides to take up juggling to perform at the school talent show a month in the future. He starts off practicing juggling 3 balls, and slowly gets better adding 1 ball to his juggling act each week. After the end of the fourth week the talent show begins, but when {first_name} walks on stage he slips and drops three of his balls. 2 of them are caught by people in the crowd as they roll off the stage, but one gets lost completely since the auditorium is dark. With a sigh, {first_name} starts to juggle on stage with how many balls?"], "ground_truth": 4,},

}
\end{lstlisting}

\begin{lstlisting}[caption={\textbf{Intervenability dataset for animals}},label={lst:intervdatasetanimals}]
"0": {"sentence": [f"Out of the 60 {animal}s in a kennel, 9 {animal}s enjoy watermelon, 48 {animal}s enjoy salmon, and 5 {animal}s enjoy both watermelon and salmon. How many {animal}s in the kennel do not eat either watermelon or salmon?"], "ground_truth": 8},
"1": {"sentence": [f"James needs to get more toys for his {animal} shelter. Each {animal} needs one toy. James currently has 4 toys on hand for 4 {animal}s, but there are 8 more {animal}s in the shelter now. After buying the toys, he went back to see that there are twice as many more {animal}s than when he left so he had to buy some more toys. When James came back yet again, 3 {animal}s were gone so he no longer needed those toys. How many toys in total does James need?"], "ground_truth": 33},
"2": {"sentence": [f"If Elise purchased a 15kg bag of {animal} food and then another 10kg bag, resulting in a total of 40kg of {animal} food, how many kilograms of {animal} food did Elise already have before her recent purchases?"], "ground_truth": 15},
"3": {"sentence": [f"John takes care of 10 {animal}s. Each {animal} takes .5 hours a day to walk and take care of their business. How many hours a week does he spend taking care of {animal}s?"], "ground_truth": 35},
"4": {"sentence": [f"Cecilia just bought a new {animal}. According to her veterinarian, she has to feed the {animal} 1 cup of {animal} food every day for the first 180 days. Then she has to feed the {animal} 2 cups of {animal} food every day for the rest of its life. If one bag of {animal} food contains 110 cups, how many bags of {animal} food will Cecilia use in the first year?"], "ground_truth": 5},
"5": {"sentence": [f"What is the total weight, in ounces, of the pet food that Mrs. Anderson bought if she purchased 2 bags of 3-pound cat food and 2 bags of {animal} food that weigh 2 pounds more than each bag of cat food?"], "ground_truth": 256},
"6": {"sentence": [f"If Belle consumes 4 {animal} biscuits and 2 rawhide bones every evening, and each rawhide bone costs $1 and each {animal} biscuit costs $0.25, what is the total cost, in dollars, to feed Belle these treats for a week?"], "ground_truth": 21},
"7": {"sentence": [f"If Ben's potato gun can launch a potato 6 football fields, and each football field is 200 yards long, and Ben's {animal} can run at a speed of 400 feet per minute, how many minutes will it take for his {animal} to fetch a potato that he launches?"], "ground_truth": 9},
"8": {"sentence": [f"John adopts a {animal}. He takes the {animal} to the groomer, which costs $100. The groomer offers him a 30% discount for being a new customer. How much does the grooming cost?"], "ground_truth": 70},
"9": {"sentence": [f"It takes 2.5 hours to groom a {animal} and 0.5 hours to groom a cat. What is the number of minutes it takes to groom 5 {animal}s and 3 cats?"], "ground_truth": 840},
"10": {"sentence": [f"A cat has nine lives.  A {animal} has x less lives than a cat.  A mouse has 7 more lives than a {animal}. A mouse has 13 lives. What is the value of unknown variable x?"], "ground_truth": 3},
"11": {"sentence": [f"If seven more {animal}s are added to the thirteen in the cage, the number of {animal}s in the cage will be 1/3 the number of {animal}s Jasper saw in the park today. How many {animal}s did Jasper see in the park today?"], "ground_truth": 60},
"12": {"sentence": [f'Chris has a two-speed lawn mower. He can mow his entire lawn in "turtle" mode in 1 hour, or 40 minutes in "{animal}" mode. Today, he experimented by mowing half in turtle mode and half in {animal} mode. How many minutes did it take him to mow the lawn?'], "ground_truth": 50},
"13": {"sentence": [f"At a James Bond movie party, each guest is either male (M) or female (F. 40% of the guests are women, 80% of the women are wearing {animal} ears, and 60% of the males are wearing {animal} ears. If the total number of guests at the party is 200, given all this information, what is the total number of people wearing {animal} ears?"], "ground_truth": 136},
"14": {"sentence": [f"Every month, Madeline has to buy food, treats, and medicine for her {animal}. Food costs $25 per week. Treats cost $20 per month. Medicine costs $100 per month. How much money does Madeline spend on her {animal} per year if there are 4 weeks in a month?"], "ground_truth": 2640},
"15": {"sentence": [f"Russell works at a pet store and is distributing straw among the rodents. The rats are kept in 3 cages in equal groups and each rat is given 6 pieces of straw. There are 10 cages of hamsters that are kept alone and each hamster is given 5 pieces of straw. There is also a pen of {animal} where 20 pieces of straw are distributed among the {animal}. No straw is used anywhere else in the store. If 160 pieces of straw have been distributed among the small rodents, how many rats are in each cage?"], "ground_truth": 5},
\end{lstlisting}

\begin{lstlisting}[caption={\textbf{Intervenability dataset for locations}},label={lst:intervdatasetlocation}]
{
"0": {"sentence": [f"If the total cost of all the {location} supplies that Pauline wants to purchase is $150 before sales tax, and the sales tax is 8% of the total amount purchased, what will be the total amount that Pauline will spend on all the items, including sales tax?"], "ground_truth": 162},
"1": {"sentence": [f"If adding seven more rabbits to the thirteen already in the cage would result in a total number of rabbits that is 1/3 of the number of rabbits Jasper saw in the {location} today, what is the total number of rabbits that Jasper saw in the {location} today?"], "ground_truth": 60},
"2": {"sentence": [f"If there were initially 250 books in the {location} and 120 books were taken out on Tuesday to be read by children, followed by 35 books being returned on Wednesday and another 15 books being withdrawn on Thursday, what is the current total number of books in the {location}?"], "ground_truth": 150},
"3": {"sentence": [f"Jason goes to the {location} 4 times more often than William goes. If William goes 2 times per week to the {location}, how many times does Jason go to the {location} in 4 weeks?"], "ground_truth": 32},
"4": {"sentence": [f"If there were initially 235 books in the {location} and 227 books were borrowed on Tuesday, then 56 books were returned on Thursday, and 35 books were borrowed again on Friday, how many books are currently in the {location}?"], "ground_truth": 29},
"5": {"sentence": [f"A {location} has a number of books. 35% of them are intended for children and 104 of them are for adults. How many books are in the {location}?"], "ground_truth": 160},
"6": {"sentence": [f"If Benjamin spent a total of $15 at {location}, buying a salad, a burger, and two packs of fries, and one pack of fries cost $2, with the salad being three times that price, what is the cost of the burger?"], "ground_truth": 5},
}
\end{lstlisting}

\newpage
\subsection{Sweep results}
\begin{table}[h!]
\centering
\caption{Sweep results for different insertion values for intervening on the \texttt{animals} concept}
\begin{tabular}{lccccc}
\hline
metric & insertion value & acc $(10^{-4})$ & acc $(10^{-5})$ & acc $(10^{-6})$ & acc $(0)$ \\
\hline
eval & 10 & 0.784 & 0.628 & 0.625 & 0.600 \\
eval & 20 & 0.769 & 0.662 & 0.644 & 0.619 \\
eval & 50 & 0.716 & 0.650 & 0.637 & 0.619 \\
eval & 100 & 0.709 & 0.669 & 0.659 & 0.631 \\
eval & 150 & 0.716 & 0.662 & 0.656 & 0.647 \\
eval & 200 & 0.725 & 0.656 & 0.666 & 0.628 \\
eval & 300 & 0.731 & 0.656 & 0.650 & 0.637 \\
eval & 500 & 0.662 & 0.578 & 0.647 & 0.597 \\
& & & & & \\
include & 10 & 0.487 & 0.572 & 0.656 & 0.653 \\
include & 20 & 0.734 & 0.744 & 0.772 & 0.781 \\
include & 50 & 0.912 & 0.925 & 0.934 & 0.922 \\
include & 100 & 0.938 & 0.928 & 0.931 & 0.916 \\
include & 150 & 0.934 & 0.925 & 0.925 & 0.916 \\
include & 200 & 0.934 & 0.916 & 0.912 & 0.906 \\
include & 300 & 0.912 & 0.831 & 0.816 & 0.797 \\
include & 500 & 0.678 & 0.547 & 0.531 & 0.544 \\
& & & & & \\
drop & 10 & 0.344 & 0.241 & 0.153 & 0.147 \\
drop & 20 & 0.138 & 0.113 & 0.084 & 0.078 \\
drop & 50 & 0.028 & 0.009 & 0.000 & 0.003 \\
drop & 100 & 0.000 & 0.000 & 0.000 & 0.000 \\
drop & 150 & 0.000 & 0.000 & 0.000 & 0.000 \\
drop & 200 & 0.000 & 0.000 & 0.000 & 0.000 \\
drop & 300 & 0.000 & 0.000 & 0.000 & 0.000 \\
drop & 500 & 0.003 & 0.000 & 0.009 & 0.006 \\
\hline
\label{tab:animals}
\end{tabular}
\end{table}

\begin{table}[h!]
\centering
\caption{Sweep results for different insertion values for intervening on the \texttt{locations} concept}
\begin{tabular}{lccccc}
\hline
metric & insertion value & acc $(10^{-4})$ & acc $(10^{-5})$ & acc $(10^{-6})$ & acc $(0)$ \\
\hline
eval & 10 & 0.957 & 0.879 & 0.893 & 0.929 \\
eval & 20 & 0.957 & 0.864 & 0.914 & 0.929 \\
eval & 50 & 0.957 & 0.871 & 0.921 & 0.914 \\
eval & 100 & 0.957 & 0.871 & 0.907 & 0.907 \\
eval & 150 & 0.929 & 0.864 & 0.907 & 0.921 \\
eval & 200 & 0.914 & 0.886 & 0.900 & 0.914 \\
eval & 300 & 0.893 & 0.843 & 0.907 & 0.871 \\
eval & 500 & 0.814 & 0.779 & 0.893 & 0.821 \\
& & & & & \\
include & 10 & 0.229 & 0.214 & 0.321 & 0.421 \\
include & 20 & 0.400 & 0.429 & 0.464 & 0.429 \\
include & 50 & 0.521 & 0.586 & 0.607 & 0.579 \\
include & 100 & 0.557 & 0.586 & 0.600 & 0.586 \\
include & 150 & 0.557 & 0.586 & 0.600 & 0.593 \\
include & 200 & 0.564 & 0.593 & 0.593 & 0.600 \\
include & 300 & 0.571 & 0.600 & 0.586 & 0.593 \\
include & 500 & 0.564 & 0.593 & 0.586 & 0.657 \\
& & & & & \\
drop & 10 & 0.379 & 0.400 & 0.393 & 0.243 \\
drop & 20 & 0.243 & 0.236 & 0.221 & 0.186 \\
drop & 50 & 0.107 & 0.057 & 0.021 & 0.000 \\
drop & 100 & 0.029 & 0.000 & 0.000 & 0.000 \\
drop & 150 & 0.007 & 0.000 & 0.000 & 0.000 \\
drop & 200 & 0.000 & 0.000 & 0.000 & 0.000 \\
drop & 300 & 0.000 & 0.000 & 0.000 & 0.000 \\
drop & 500 & 0.014 & 0.007 & 0.000 & 0.000 \\
\hline
\label{tab:location}
\end{tabular}
\end{table}

\end{document}